\documentclass[10pt]{article}
\usepackage{fullpage}
\usepackage{amsthm}
\newtheorem*{rep@theorem}{\rep@title}
\newenvironment{oneshot}[1]{\def\rep@title{#1} \begin{rep@theorem}}{\end{rep@theorem}}
\usepackage{multicol}
\usepackage{natbib}
\setcitestyle{authoryear,round,citesep={;},aysep={,},yysep={;}}

\renewcommand{\cite}[1]{\citep{#1}}

\usepackage{graphicx,subfigure,color,tikz,multirow,rotating,caption}

\usepackage{algorithm,algorithmic}

\usepackage{Definitions}

\def\shownotes{0}  
\ifnum\shownotes=1
\newcommand{\authnote}[2]{{$\ll$\textsf{\footnotesize #1 notes: #2}$\gg$}}
\else
\newcommand{\authnote}[2]{}
\fi

\newtheorem{thm}{Theorem}[section]

\newtheorem{cor}[thm]{Corollary}
\newtheorem{lem}[thm]{Lemma}

\numberwithin{equation}{section}

\newcommand{\E}{\mathbb{E}}
\newcommand{\Exp}{\mathop{\mathbb E}\displaylimits}

\newcommand{\R}{\mathbb{R}}

\newcommand{\LengthDocument}{L}

\newcommand{\expinner}[1]{\exp(\langle #1\rangle)}

\newcommand{\mul}{\textsf{Mul}}
\newcommand{\indicator}[1]{\mathbf{1}_{#1}}

\title{RAND-WALK: A latent variable model approach to word embeddings}

\author{
Sanjeev Arora
\and Yuanzhi Li
\and Yingyu Liang
\and Tengyu Ma
\and Andrej Risteski \thanks{Princeton University, Computer Science Department. \texttt{\{arora,yuanzhil,yingyul,tengyu,risteski\}@cs.princeton.edu}. This work was supported in part by NSF grants CCF-0832797, CCF-1117309, CCF-1302518, DMS-1317308, Simons Investigator Award, and Simons Collaboration Grant. Tengyu Ma was also supported by Simons Award for Graduate Students in Theoretical Computer Science and IBM PhD Fellowship. }
}

\date{}

\begin{document}
\maketitle


\newcommand{\newcite}[1]{\citet{#1}} 

\begin{abstract}
Semantic word embeddings represent  the meaning of a word via a vector, and are created by diverse methods. Many use nonlinear operations on co-occurrence statistics, and have hand-tuned hyperparameters and reweighting methods.

This paper proposes a new generative model, a dynamic version of the log-linear topic model of~\citet{mnih2007three}.  The methodological novelty is to use the prior to  compute  closed form expressions for word statistics. This provides a theoretical justification for nonlinear models like PMI, word2vec, and GloVe, as well as some hyperparameter choices. It also helps explain why low-dimensional semantic embeddings contain linear algebraic structure that allows solution of word analogies, as shown by~\citet{mikolov2013efficient} and many subsequent papers.
 
Experimental support is provided for the generative model assumptions, the most important of which is  that latent word vectors are fairly uniformly dispersed in space. 
 
\end{abstract}

%
%
 %
 %

\section{Introduction} \label{sec:intro}
Vector representations of words (word embeddings) try to capture relationships between words as distance or angle, and have many applications in computational linguistics and machine learning. 
They are constructed by various models whose unifying  philosophy is that the meaning of a word is defined by  \textquotedblleft the company it keeps\textquotedblright~\cite{firth1957a}, namely, co-occurrence statistics. The simplest methods use word vectors that explicitly represent co-occurrence statistics.
Reweighting heuristics are known to improve these methods, as is dimension reduction~\cite{deerwester1990indexing}.  Some reweighting methods are nonlinear, which include taking the square root of co-occurrence counts~\cite{rohde2006}, or the logarithm, or 
the related Pointwise Mutual Information (PMI)~\cite{church1990word}.
These are collectively referred to as Vector Space Models, surveyed in~\cite{turney2010frequency}.

Neural network language models~\cite{hinton1984distributed,rumelhart1988learning,bengio2006neural,collobert2008unified} propose another way to construct embeddings: the word vector is simply the neural network's internal representation for the word. This method is nonlinear and nonconvex. It was popularized via
word2vec, a family of energy-based models in~\cite{mikolov2013distributed,mikolov2013linguistic}, followed by a matrix factorization approach called GloVe~\cite{pennington2014glove}.
The first paper also showed how to solve analogies using linear algebra on word embeddings.
Experiments and theory were used to suggest that these newer methods
are related to the older PMI-based models, but with new hyperparameters and/or
term reweighting methods~\cite{levy2014neural}. 

But note that even the old PMI method is a bit mysterious. The simplest version considers a symmetric matrix with each row/column indexed by a word. The entry for $(w, w')$ 
is $\mbox{PMI}(w, w') =\log \frac{p(w, w')}{p(w)p(w')}$, where 
$p(w, w')$ is the empirical
probability  of words $w, w'$ appearing within a window of
certain size in the corpus,  and $p(w)$ is the marginal probability of $w$. (More complicated models could use 
asymmetric matrices with columns corresponding to context words or phrases, and also involve tensorization.)
Then word vectors are obtained by low-rank SVD on this matrix, or a related matrix with term reweightings.   
In particular, the PMI matrix is found to be closely approximated by a low rank matrix:
there exist word vectors in say $300$ dimensions, which is much smaller than the number of words in the dictionary, such that
\begin{align} \label{eqn:whypmi}
\inner{v_w}{v_{w'}} \approx \mbox{PMI}(w, w')
\end{align}
where $\approx$ should be interpreted loosely.

There appears to be no theoretical explanation for this empirical
finding about the approximate low rank of the PMI matrix. The current paper addresses this.
Specifically, we propose a probabilistic model of text generation that augments the log-linear topic model of~\citet{mnih2007three} with dynamics, in the form of
a random walk over a latent discourse space.  
The chief methodological contribution is using the model priors to analytically derive a closed-form expression that directly explains~(\ref{eqn:whypmi}); see Theorem~\ref{thm:main} in Section~\ref{sec:model}. 
Section~\ref{sec:unification} builds on this insight to give a rigorous justification for
models such as word2vec and GloVe, including the hyperparameter choices for the latter.  
The insight also leads to a mathematical explanation for why these word embeddings allow analogies to be solved using linear algebra; see Section~\ref{sec:interpretability}.
Section~\ref{sec:experiment} shows good empirical fit to this model's assumtions and predictions, including the surprising one that word vectors are pretty uniformly distributed (isotropic) in space.%
\footnote{The code is available at \url{https://github.com/PrincetonML/SemanticVector} }

\subsection{Related work}

Latent variable probabilistic models of language have been used for word embeddings before, including
Latent Dirichlet Allocation (LDA) and its more complicated variants (see the survey~\cite{blei2012probabilistic}), and some neurally inspired nonlinear models~\cite{mnih2007three,DBLP:conf/acl/MaasDPHNP11}. In fact, LDA evolved
out of efforts in the 1990s to provide a generative model that \textquotedblleft explains\textquotedblright\ 
the success of older vector space methods like Latent Semantic Indexing~\cite{Papadimitriouetal98,hofmann1999probabilistic}. However, none of these earlier generative models
has been linked to PMI models.

\citet{levy2014neural} tried to relate word2vec to PMI models. They showed that if there were no dimension constraint in word2vec, specifically, the \textquotedblleft skip-gram with negative sampling (SGNS)\textquotedblright\ version of the model, then its solutions would
satisfy (\ref{eqn:whypmi}), provided 
the right hand side
were replaced by $\mbox{PMI}(w, w') -\beta$ for some scalar $\beta$.
However, skip-gram is a discriminative model (due to the use of negative sampling), not generative. Furthermore,  their argument only applies to very high-dimensional word embeddings, and thus does not address low-dimensional embeddings, which have superior quality in applications.

\citet{hashimoto2016word} focuses on issues similar to our paper. They model text generation as a random walk on words, which are assumed to be embedded as vectors in a geometric space. Given that the last word produced was $w$, the probability that the next word is $w'$ is assumed to be given by
$h(|v_w -v_{w'}|^2)$ for a suitable function $h$,  and this model leads to an explanation of 
(\ref{eqn:whypmi}). 
 By  contrast our random walk involves a latent discourse vector, which has a clearer semantic interpretation and has proven useful in subsequent work, e.g. understanding structure of word embeddings for polysemous words~\newcite{arora2016linear}. Also our work clarifies 
 some weighting and bias terms in the training objectives of previous methods (Section~\ref{sec:unification})
and also the phenomenon discussed in the next paragraph. 

Researchers have tried to understand why vectors obtained from the highly nonlinear word2vec models exhibit linear structures~\cite{levy2014linguistic,pennington2014glove}. Specifically, for analogies like ``\emph{man}:\emph{woman}::\emph{king}:{??},"  \emph{queen} happens to be the word whose vector  $v_{queen}$  is the most similar to the vector $v_{king} -v_{man} + v_{woman}$. This suggests that simple semantic relationships, such as
{\em masculine} vs  {\em feminine} tested in the above example, correspond approximately to a single direction in space, a phenomenon we will henceforth refer to as {\sc relations=lines}.

Section~\ref{sec:interpretability} surveys earlier attempts to explain this phenomenon and their shortcoming, namely, that they ignore the large approximation error in relationships like (\ref{eqn:whypmi}). This error appears larger than the  
difference between the best solution and the second best (incorrect) solution in analogy solving, so that
this error could in principle lead to a complete failure in analogy solving. In our explanation, the low dimensionality 
of the word vectors plays a key role. 
This can also be seen as a theoretical 
explanation of the old observation
that dimension reduction improves the quality of word embeddings for various tasks.
The intuitive explanation often given ---that smaller models generalize better---turns out to be fallacious, since the training method for creating embeddings makes no reference to analogy solving.
Thus there is no a priori reason why low-dimensional model parameters (i.e., lower model capacity) should lead to better performance in analogy solving, just as there is no reason they are better at some other unrelated task like predicting the weather.

\subsection{Benefits of generative approaches}
\label{subsec:genbenefit}
In addition to giving some form of \textquotedblleft unification\textquotedblright\ of existing methods, our generative model also brings more intepretability to word embeddings beyond traditional cosine similarity and even analogy solving.  For example, it led to an understanding of how
the different senses of a polysemous word  (e.g., {\em bank}) reside in linear superposition within the word embedding~\cite{arora2016linear}.  Such insight into embeddings may prove useful in the numerous settings in NLP and neuroscience where they are used.

Another new explanatory feature of our model is that low dimensionality of word embeddings plays a key theoretical role ---unlike in previous papers where the model is agnostic about the dimension of the embeddings, and the superiority of low-dimensional embeddings is an empirical finding
(starting with~\newcite{deerwester1990indexing}). Specifically, our theoretical analysis makes the key assumption that the set of all word vectors
(which are latent variables of the generative model) are spatially isotropic, which means that
they have no preferred direction in space.  Having $n$ vectors be isotropic in $d$ dimensions 
requires $d \ll n$.
This isotropy is needed in the calculations (i.e., multidimensional integral) that yield~(\ref{eqn:whypmi}). It also holds empirically for our word vectors, as shown in 
Section~\ref{sec:experiment}.  

The isotropy of low-dimensional word vectors also plays a key role in our explanation of the {\sc relations=lines} phenomenon~(Section~\ref{sec:interpretability}). The  isotropy  has  a \textquotedblleft purification\textquotedblright\ effect that  mitigates the effect 
of the (rather large) approximation error in the PMI models.




\section{Generative model and its properties}\label{sec:model}

\newtheorem{model}{Model}
\newcommand{\tv}{\tilde{v}}
\newcommand{\consWordVector}{\Gamma}

The model treats corpus generation as a dynamic process, where the $t$-th word is produced at step $t$. The process is driven by the random walk of a discourse vector $c_t \in \Re^d$. Its coordinates represent what is being talked about.\footnote{This is a different interpretation of the term \textquotedblleft discourse\textquotedblright\ compared to some other settings in computational linguistics.} Each word has a  (time-invariant) latent vector $v_w \in \Re^d$ that captures its correlations with the discourse vector.
We model this bias with a log-linear word production model: 
\begin{equation} 
\displaystyle 
 \Pr[w~\mbox{emitted at time $t$}~|~c_t] \propto \expinner{c_t,v_w}. \label{eqn:themodel}
\end{equation}

The discourse vector $c_t$ does a slow random walk  (meaning that $c_{t+1}$ is obtained from $c_t$ by adding a small random displacement vector), so that nearby words are generated under similar discourses.
We are interested in the probabilities that word pairs co-occur near each other, so occasional big jumps in the random walk are allowed because they have negligible effect on these  probabilities.
 
A similar log-linear model appears in~\citet{mnih2007three} but without the random walk. 
The linear chain CRF of~\citet{lafferty2001conditional} is more general.
The dynamic topic model of~\citet{blei2006dynamic} utilizes topic dynamics, but with a linear word production model. \citet{belanger2015linear} have proposed a dynamic model for text using Kalman Filters, where the sequence of words is generated from  Gaussian linear dynamical systems, rather than the log-linear model in our case.

The novelty here over such past works is a theoretical analysis  in the method-of-moments tradition~\cite{hsu2012spectral,cohen2012spectral}. 
Assuming a prior on the random walk we analytically integrate out the hidden random variables and compute a simple closed form expression that approximately connects the model parameters to the observable joint probabilities (see Theorem~\ref{thm:main}). This is reminiscent of analysis of similar random walk models in finance~\cite{black1973pricing}.

\paragraph{Model details.}  
Let $n$ denote the number of words and $d$ denote the dimension of the discourse space, where $1\le d\le n$.  
Inspecting (\ref{eqn:themodel}) suggests word vectors need to have  varying lengths, to fit the empirical finding that word probabilities satisfy a power law. Furthermore, we will assume that 
in the bulk, the word vectors are distributed uniformly in space, earlier referred to as isotropy. 
This can be quantified as a prior in the Bayesian tradition. More precisely, the ensemble of word vectors consists of i.i.d draws generated by $v =  s \cdot \hat{v}$, where $\hat{v}$ is from the spherical Gaussian distribution, and $s$ is a scalar random variable.
We assume $s$ is a random scalar with expectation $\tau = \Theta(1)$ and 
 $s$ is always upper bounded by $\kappa $, which is another constant. Here $\tau$ governs the expected magnitude of $\inner{v}{c_t}$, 
 and it is particularly important to choose it to be $\Theta(1)$ so that the distribution $\Pr[w\vert c_t] \propto \exp(\inner{v_w}{c_t})$ is interesting.\footnote{A larger $\tau$ will make $\Pr[w\vert c_t]$ too peaked and a smaller one will make it too uniform.} Moreover, the dynamic range of word probabilities will roughly equal $\exp(\kappa^2)$, so one should think of $\kappa$ as an absolute constant like $5$. These details about $s$ are important for realistic modeling but not too important in our analysis. (Furthermore, readers uncomfortable with this simplistic Bayesian prior should look at Section~\ref{sec:weakassumption} below.)

Finally, we clarify the nature of the random walk. We assume that the stationary distribution of the random walk is uniform over the unit sphere, denoted by $\mathcal C$. 
The transition kernel of the random walk can be in any form so long as 
at each step the movement of the discourse vector is at most $\epsilon_2/\sqrt{d}$
in $\ell_2$ norm.\footnote{ 
	More precisely, the proof extends to any symmetric product stationary distribution $\mathcal C$ with sub-Gaussian coordinate satisfying $\Exp_{c}\left[\|c\|^2\right] = 1$, and the steps are such that for all $c_t$, $\Exp_{p(c_{t+1}|c_t)}[\exp(\kappa\sqrt{d}\|c_{t+1}-c_t\|)]\le 1+\epsilon_2$ for some small $\epsilon_2$.
}
This is still fast enough to let the walk mix quickly in the space.  

The following lemma (whose proof appears in the appendix) is central to the analysis. It says that under the Bayesian prior, 
the partition function $Z_c = \sum_{w} \exp(\inner{v_w}{c})$, which is the implied normalization in equation~\eqref{eqn:themodel}, is close to some constant $Z$ for most of the discourses $c$.    
This can be seen as a plausible theoretical explanation of a phenomenon called {\em self-normalization} in log-linear models:  ignoring the partition function or treating it as a constant (which greatly simplifies training) is known to often give good results.  This has also been studied in~\cite{andreas2014when}.

\begin{lem}[Concentration of partition functions]\label{thm:Z_c}
	If the word vectors satisfy the Bayesian prior described in the model details, then 
	\begin{equation}
		\Pr_{c\sim \mathcal{C}}\left[(1-\epsilon_z)Z\le Z_c\le (1+\epsilon_z) Z\right]\ge 1-\delta \label{eqn:lem_zc},
	\end{equation}
	%
	%
	for $\epsilon_z = \widetilde{O}(1/\sqrt{n})$, and $\delta = \exp(-\Omega(\log^2 n)) $.
\end{lem}

The concentration of the partition functions then leads to our main theorem (the proof is in the appendix).
The theorem gives simple closed form approximations for $p(w)$, the  probability of word $w$ in the corpus, and $p(w,w')$, the probability that two words $w,w'$ occur next to each other. The theorem states the result for  the window size $q = 2$,  but the same analysis works for pairs that appear in a small window, say of size $10$, as stated in Corollary~\ref{cor:qsize}. 
Recall that $\mbox{PMI}(w, w') = \log [p(w,w')/(p(w)p(w'))]$. 
  
  \begin{thm}\label{thm:main}
  	\label{l:pairs} 
  	Suppose the word vectors satisfy the inequality~\eqref{eqn:lem_zc}, and window size $q = 2$. Then, 
\begin{align}
	\log p(w,w')   & =  \frac{\|v_w + v_{w'}\|_2^2 }{2d}- 2\log Z  \pm \epsilon,  \label{eqn:twowords_cooc} \\
	\log p(w) & =  \frac{\nbr{v_{w}}_2^2}{2d}  - \log Z \pm \epsilon. \label{eqn:singleword}
\end{align}
for $\epsilon = O(\epsilon_z) +\widetilde{O}(1/d) + O(\epsilon_2)$. 
  Jointly these imply: 
\begin{equation}
	\mbox{PMI}~(w, w') = \frac{\langle v_w, v_{w'}\rangle}{d}  \pm O(\epsilon). \label{eqn:PMI}
\end{equation}
  \end{thm}

\vspace{2mm}
\noindent
{\bf Remarks 1.} Since the word vectors have $\ell_2$ norm of the order of $\sqrt{d}$,  for two typical word vectors $v_w,v_{w'}$, $\|v_{w}+ v_{w'}\|_2^2$ is of the order of $\Theta(d)$. Therefore the noise level $\epsilon$ is very small compared to the leading term $\frac{1}{2d}\|v_w + v_{w'}\|_2^2$. For PMI however, the noise level $O(\epsilon)$ could be comparable to the leading term, and empirically we also find higher error here.

\vspace{2mm}
\noindent
{\bf Remarks 2.} 
Variants of the expression for joint probability in  (\ref{eqn:twowords_cooc}) had been hypothesized based upon empirical evidence in~\newcite{mikolov2013distributed} and also~~\newcite{globerson2007euclidean}, and~\newcite{DBLP:conf/nips/MaronLB10} . 

\vspace{2mm}
\noindent
{\bf Remarks 3.}
Theorem~\ref{thm:main} directly leads to the extension to a general window size $q$ as follows: 
\begin{cor}\label{cor:qsize}
	Let $p_q(w,w')$ be the co-occurrence probability in windows of size $q$, and $\mbox{PMI}_q(w,w')$ be the corresponding PMI value. Then 
	\begin{align}
	\log p_q(w,w')   & =  \frac{\|v_w + v_{w'}\|_2^2 }{2d}- 2\log Z  + \gamma \pm \epsilon,  \nonumber\\
\mbox{PMI}_{q}~(w, w') &= \frac{\langle v_w, v_{w'}\rangle}{d}  + \gamma \pm O(\epsilon). \nonumber
	\end{align}
	where $\gamma = \log \left(\frac{q(q-1)}{2}\right)$. 
\end{cor}


It is quite easy to see that Theorem~\ref{thm:main} implies the Corollary~\ref{cor:qsize}, as when the window size is $q$ the pair $w,w'$ could appear in any of ${q \choose 2}$ positions within the window, and the joint probability of $w,w'$ is roughly the same for any positions because the discourse vector changes slowly. (Of course, the error term gets worse as we consider larger window sizes, although for any constant size, the statement of the theorem is correct.) This is also consistent with the shift $\beta$ for fitting PMI in~\cite{levy2014neural}, which showed that  
without dimension constraints, the solution to skip-gram with negative sampling satisfies $\mbox{PMI}~(w, w') - \beta =  \langle v_w, v_{w'}\rangle$ for a constant $\beta$ that is related to the negative sampling in the optimization. Our result justifies via a generative model why this should be satisfied even for low dimensional word vectors.



\subsection{Proof sketches}
Here we provide the proof sketches, while the complete proof can be found in the appendix. 

\paragraph{Proof sketch of Theorem~\ref{thm:main}}
Let $w$ and $w'$ be two arbitrary words. Let $c$ and $c'$ denote two consecutive context vectors, where $c\sim \mathcal{C}$ and $c'\vert c$ is defined by the Markov kernel $p(c'\mid c)$.  

We start by using the law of total expectation, integrating out the hidden variables $c$ and $c'$:
\begin{align}
 p(w,w')  & =  \Exp_{c,c'}\left[\Pr[w,w'\vert c,c']\right] \nonumber\\
 & = \Exp_{c,c'}[p(w\vert c)p(w'\vert c')] \nonumber\\
 & = \Exp_{c,c'}\left[  \frac{\expinner{v_w,c}}{Z_c}\frac{\expinner{v_{w'},c'}}{Z_{c'}} \right] \label{eqn:int}
\end{align}

An expectation like~\eqref{eqn:int} would normally be difficult to analyze because of the partition functions. However, we can assume the inequality~\eqref{eqn:lem_zc}, that is, the partition function typically does not vary much for most of context vectors $c$. Let $F$ be the event that both $c$ and $c'$ are within $(1\pm \epsilon_z)Z$. Then by~\eqref{eqn:lem_zc} and the union bound, event $F$ happens with probability at least $1-2\exp(-\Omega(\log^2 n))$. We will split the right-hand side (RHS) of~\eqref{eqn:int} into the parts according to whether $F$ happens or not.  
%
\begin{align}
\textup{RHS of}~\eqref{eqn:int} &= \underbrace{\Exp_{c,c'}\left[  \frac{\expinner{v_w,c}}{Z_c}\frac{\expinner{v_{w'},c'}}{Z_{c'}}\mathbf{1}_{F} \right]}_{T_1}\nonumber\\
&+ \underbrace{\Exp_{c,c'}\left[  \frac{\expinner{v_w,c}}{Z_c}\frac{\expinner{v_{w'},c'}}{Z_{c'}} \mathbf{1}_{\bar{F}}\right]}_{T_2}\label{eqn:eqn1}
\end{align}
where $\bar{F}$ denotes the complement of event $F$ and $\mathbf{1}_{F}$ and $\mathbf{1}_{\bar{F}}$ denote indicator functions for $F$ and $\bar{F}$, respectively. When $F$ happens, we can replace $Z_c$ by $Z$ with a $1\pm \epsilon_z$ factor loss: The first term of the RHS of~\eqref{eqn:eqn1} equals to
\begin{align}
 T_1 = \frac{1\pm O(\epsilon_z)}{Z^2}\Exp_{c,c'}\left[ \expinner{v_w,c}\expinner{v_{w'},c'}\mathbf{1}_{F} \right] \label{eqn:eqn2}
\end{align}

On the other hand, we can use $\Exp[\mathbf{1}_{\bar{F}}] = \Pr[\bar{F}] \le \exp(-\Omega(\log^2 n))$ to show that the second term of RHS of~\eqref{eqn:eqn1} is negligible,  
\begin{equation}
|T_2| = \exp(-\Omega(\log^{1.8} n))\,. \label{eqn:eqn5}
\end{equation}
This claim must be handled somewhat carefully since the RHS does not depend on $d$ at all. 
Briefly, the reason this holds is as follows: in the regime when $d$ is small ($\sqrt{d} = o(\log^2 n)$), any word vector $v_w$ and discourse $c$ satisfies that $\exp(\langle v_w, c\rangle) \leq \exp(\|v_w\|) = \exp(O(\sqrt{d}))$, and since $\Exp[\mathbf{1}_{\bar{F}}] = \exp(-\Omega(\log^2 n)) $, the claim follows directly; In the regime when $d$ is large ($\sqrt{d} = \Omega(\log^2 n)$), we can use concentration inequalities to show that except with a small probability $\exp(-\Omega(d))= \exp(-\Omega(\log^2 n))$, a uniform sample from the sphere behaves equivalently to sampling all of the coordinates from a standard Gaussian distribution with mean 0 and variance $\frac{1}{d}$, in which case the claim is not too difficult to show using Gaussian tail bounds. 
   
Therefore it suffices to only consider~\eqref{eqn:eqn2}. Our model assumptions state that $c$ and $c'$ cannot be too different. We leverage that by rewriting~\eqref{eqn:eqn2} a little, and get that
it equals
\begin{align}
T_1&= \frac{1 \pm O(\epsilon_z)}{Z^2} \Exp_c\left[\expinner{v_w,c}\Exp_{c'\mid c}\left[\expinner{v_{w'},c'} \right]\right] \nonumber \\
& = \frac{1 \pm O(\epsilon_z)}{Z^2} \Exp_c\left[\expinner{v_w,c}A(c)\right] \label{eqn:eqn3}
\end{align}
where $\displaystyle A(c) :=\Exp_{c'\mid c}\left[\expinner{v_{w'},c'} \right]$. We claim that $\displaystyle A(c) = (1 \pm O(\epsilon_2)) \expinner{v_{w'},c}$.
Doing some algebraic manipulations,  
\begin{align*}
A(c) 
&=\expinner{v_{w'},c}\Exp_{c'\mid c}\left[\expinner{v_{w'},c'-c} \right].
\end{align*}
Furthermore,  by our model assumptions, $\|c-c'\| \le \epsilon_2/\sqrt{d}$.  So 
$$ \langle v_{w},c- c'\rangle \le \|v_w\| \|c-c'\| = O(\epsilon_2)$$ 
and thus $\displaystyle A(c) = (1 \pm O(\epsilon_2)) \expinner{v_{w'},c}$. 
Plugging the simplification of $A(c)$ to~\eqref{eqn:eqn3}, 
\begin{eqnarray}
T_1 = \frac{1 \pm O(\epsilon_z)}{Z^2}\Exp[\expinner{v_w+v_{w'},c}]. \label{eqn:eqn4}
\end{eqnarray}
Since $c$ has uniform distribution over the sphere, the random variable $\inner{v_w+v_{w'}}{c}$ has distribution pretty similar to Gaussian distribution $\mathcal{N}(0,\|v_w+v_{w'}\|^2/d)$, especially when $d$ is relatively large. Observe that $\Exp[\exp(X)]$ has a closed form for Gaussian random variable $X\sim \mathcal{N}(0,\sigma^2)$, 
\begin{align}
\Exp[\exp(X)] &= \int_x \frac{1}{\sigma\sqrt{2\pi}} \exp(-\frac{x^2}{2\sigma^2}) \exp(x)dx \nonumber\\
&= \exp(\sigma^2/2) \,.\label{eqn:gaussian}
\end{align}

Bounding the difference between $\inner{v_w+v_{w'}}{c}$ from Gaussian random variable, 
we can show that for $\epsilon = \widetilde{O}(1/d)$, 
\begin{align}
\Exp[\expinner{v_w+v_{w'},c}] = (1\pm \epsilon ) \exp\left(\frac{\|v_w+v_{w'}\|^2}{2d}\right)\,.
\label{eqn:eqn6}
\end{align}


Therefore, the series of simplification/approximation above (concretely, combining equations~\eqref{eqn:int},~\eqref{eqn:eqn1},~\eqref{eqn:eqn5},~\eqref{eqn:eqn4}, and~\eqref{eqn:eqn6})  
lead to the desired bound on $\log p(w,w')$ for the case when the window size $q = 2$.
The bound on $\log p(w)$ can be shown similarly.


\paragraph{Proof sketch of Lemma~\ref{thm:Z_c}}
Note that for fixed $c$, when word vectors have Gaussian priors assumed as in our model, $Z_c = \sum_{w} \expinner{v_w,c}$ is a sum of independent random variables. 

We first claim that using proper concentration of measure tools, it can be shown that the variance of $Z_c$ are relatively small compared to its mean $\Exp_{v_w}[Z_c]$, and thus $Z_c$ concentrates around its mean. Note this is quite non-trivial: the random variable $\expinner{v_w,c}$ is neither bounded nor subgaussian/sub-exponential, since the tail is approximately inverse poly-logarithmic instead of inverse exponential.  
In fact, the same concentration phenomenon does not happen for $w$. The occurrence probability of word $w$ is not necessarily concentrated because the $\ell_2$ norm of $v_w$ can vary a lot in our model, which allows the frequency of the words to have a large dynamic range. 

 
So now it suffices to show that $\E_{v_w}[Z_c]$ for different $c$ are close to each other.
Using the fact that the word vector directions have a Gaussian distribution, 
$\Exp_{v_w}[Z_c]$ turns out to only depend on the norm of $c$ (which is equal to 1). 
More precisely, 
\begin{equation}
\Exp_{v_w}[Z_c]  = f(\|c\|_2^2) = f(1)
\end{equation}
where $f$ is defined as $f(\alpha) = n \Exp_{s}[\exp({s}^2\alpha/2)]$ and $s$ has the same distribution as the norms of the word vectors. We sketch the proof of this. 
In our model, $v_w = s_w \cdot \hat{v}_w$, where $\hat{v}_w$ is a Gaussian vector with identity covariance $I$. 
Then
\begin{align*}
  \Exp_{v_w}[Z_c] &= n\Exp_{v_w}[\exp(\langle v_w, c\rangle)] \\
	&=n \Exp_{s_w}\left[\Exp_{v_w|s_w}\left[\exp(\langle v_w, c\rangle)\mid s_w\right]\right]
\end{align*}
where the second line is just an application of the law of total expectation, if we pick the norm of the (random) vector $v_w$ first, followed by its direction.  
Conditioned on $s_w$,  $\langle v_w,c \rangle$ is a Gaussian random variable with variance $\|c\|_2^2 s_w^2$, and therefore using similar calculation as in~\eqref{eqn:gaussian}, we have
$$\Exp_{v_w|s_w}\left[\exp(\langle v_w, c\rangle)\mid s_w\right] = \exp({s}^2\|c\|_2^2/2)\,.$$ Hence, 
$\Exp_{v_w}[Z_c] = n \Exp_{s}[\exp({s}^2\|c\|_2^2/2)]$ as needed. 

\paragraph{Proof of Theorem~\ref{thm:noise-reduction}}
The proof uses the standard analysis of linear regression. Let $V = P \Sigma Q^T$
be the SVD of $V$ and let $\sigma_1, \ldots, \sigma_d$
be the left singular values of $V$ (the diagonal entries of $\Sigma$). For notational ease we omit the subscripts in $\bar{\zeta}$ and $\zeta'$ since they are not relevant for this proof. Since $V^{\dagger} = Q\Sigma^{-1}P^T$ and thus $\bar{\zeta} = V^{\dagger}\zeta'= Q\Sigma^{-1}P^T \zeta'$, we have 
\begin{align}
\|\bar{\zeta}\|_2 \le \sigma_d^{-1}\|P^T\zeta'\|_2. \label{eqn:bzeta}
\end{align}

We claim
\begin{align}
\sigma_d^{-1} \le \sqrt{\frac{1}{c_1 n}}. \label{eqn:term1}
\end{align}
Indeed, $\sum_{i=1}^d \sigma^2_i = O(n d)$, since the average squared norm of a word vector is $d$. The claim then follows from the first assumption. 
Furthermore, by the second assumption,
$\|P^T\zeta'\|_{\infty}\le \frac{c_2}{\sqrt{n}}\|\zeta'\|_2$, so
\begin{align}
\|P^T\zeta'\|_2^2 \le \frac{c_2^2 d}{n} \|\zeta'\|_2^2. \label{eqn:term2}
\end{align}

Plugging \eqref{eqn:term1} and \eqref{eqn:term2} into \eqref{eqn:bzeta}, we get 
$$ \|\bar{\zeta}\|_2 \le \sqrt{\frac{1}{c_1 n}} \sqrt{\frac{c^2_2 d}{n}\|\zeta'\|_2^2} = \frac{c_2 \sqrt{d}}{\sqrt{c_1}n}\|\zeta'\|_2$$ 
as desired. 	The last statement follows because the norm of the signal, which is $d \log(\nu_R)$ originally and is $V^{\dagger} d \log(\nu_R) = v_a-v_b$ after dimension reduction, also gets reduced by a factor of $\sqrt{n}$.

\subsection{Weakening the model assumptions}  \label{sec:weakassumption}
For readers uncomfortable with Bayesian priors,  we can 
replace our assumptions with  concrete properties of  word vectors that are empirically 
verifiable (Section~\ref{subs: mverif}) for our final word vectors, and in fact also for word vectors computed using other recent methods.

The word meanings are assumed to be represented by some \textquotedblleft ground truth\textquotedblright\ vectors, which the experimenter is trying to recover. These ground truth vectors are assumed to be spatially isotropic in the bulk, in the following two specific ways:
(i) For almost all unit vectors $c$ the sum $\sum_{w} \expinner{v_w,c}$ is close to a constant $Z$; (ii) Singular values of the matrix of word vectors satisfy properties similar to those of random matrices, as formalized in the paragraph before Theorem~\ref{thm:noise-reduction}.
Our Bayesian prior on the word vectors happens to imply that these two conditions hold with high probability. But the conditions may hold even if the prior doesn't hold. Furthermore, they 
are compatible with all sorts of local structure among word vectors such as existence of clusterings, which would be absent in  truly random vectors drawn from our prior.




\section{Training objective and relationship to other models}
\label{sec:unification}

To get a training objective out of  Theorem~\ref{thm:main}, we reason as follows.
Let $X_{w,w'}$ be the number of times words $w$ and $w'$ co-occur within the same window in the corpus. The probability $p(w, w')$ of such a co-occurrence at any particular time is given by (\ref{eqn:twowords_cooc}). Successive samples from a random walk are not independent. But if the random walk mixes fairly quickly (the mixing time is related to the {\em logarithm} of the vocabulary size), then the distribution of $X_{w,w'}$'s is very close to a multinomial distribution $\mul(\tilde{L}, \{p(w,w')\})$, where $\tilde{L} = \sum_{w,w'} X_{w,w'}$ is the total number of word pairs.

Assuming this approximation, we show below that the maximum likelihood values for the word vectors correspond to the following optimization, 
\begin{align} 
\min_{\cbr{v_w}, C} & \sum_{w,w'} X_{w,w'} \rbr{ \log(X_{w,w'}) - \nbr{v_{w}\!+\! v_{w'}}_2^2 - C}^2 \nonumber 
\end{align}

As is usual, empirical performance is improved by weighting down very frequent word pairs,
possibly because very frequent words such as ``\emph{the}'' do not fit our model.
This is done by replacing the weighting $X_{w,w'}$ by its truncation $\min\{X_{w,w'}, X_\text{max}\}$ where $X_\text{max}$ is a constant such as $100$. 
We call this objective with the truncated weights \textbf{SN} ({\bf S}quared {\bf N}orm).

We now give its derivation.
Maximizing the likelihood of $\{X_{w,w'}\}$ is equivalent to maximizing
\begin{align*}
\ell & = \log\left(\prod_{(w, w')} p(w, w')^{X_{w,w'}} \right).
\end{align*}
Denote the logarithm of the ratio between the expected count and the empirical count as
\begin{align}
\Delta_{w,w'} = \log\left(\frac{\tilde{L} p(w, w')  }{ X_{w,w'} } \right).   \label{eqn:deltadefn}
\end{align}
Then with some calculation, we obtain the following where  $c$ is independent of the empirical observations $X_{w,w'}$'s.
\begin{align}
 \ell & =  c  +   \sum_{(w, w')}   X_{w,w'}  \Delta_{w,w'} \label{eqn:likelihood}
\end{align}

On the other hand, using $e^x \approx 1+x+x^2/2$ when $x$ is small,\footnote{This Taylor series approximation has an error of the order of $x^3$, but  ignoring it can be theoretically justified as follows. For a large $X_{w,w'}$, its value approaches its expectation and thus the corresponding $\Delta_{w,w'}$  is close to 0 and thus ignoring $\Delta_{w,w'}^3$ is well justified. The terms where $\Delta_{w,w'}$ is significant correspond to $X_{w,w'}$'s that are small.  But empirically, $X_{w,w'}$'s obey a power law distribution (see, e.g.~\newcite{pennington2014glove}) using which it can be shown that these terms  contribute a small fraction of the final objective (\ref{eqn:finalexpresion}). So we can safely ignore the errors. Full details appear in the ArXiv version of this paper~\cite{AroLiLiaMaetal15}.} we have
\begin{align*}
\tilde{L} 
& = \sum_{(w,w')} \tilde{L} p_{w,w'}  = \sum_{(w,w')} X_{w,w'} e^{\Delta_{w,w'} } \\
& \approx \sum_{(w,w')}  X_{w,w'} \left(1 + \Delta_{w,w'} + \frac{\Delta_{w,w'}^2}{2}\right).
\end{align*}
Note that $\tilde{L} = \sum_{(w,w')} X_{w,w'}$, so 
\begin{align*}
 \sum_{(w,w')}  X_{w,w'} \Delta_{w,w'}  \approx  - \frac{1}{2}\sum_{(w,w')}  X_{w,w'} \Delta_{w,w'}^2.
\end{align*}
Plugging this into (\ref{eqn:likelihood}) leads to
\begin{align}
2(c - \ell)  \approx \sum_{(w,w')}  X_{w,w'} \Delta_{w,w'}^2. \label{eqn:finalexpresion}
\end{align}
So maximizing the likelihood is approximately equivalent to minimizing the right hand side, which (by examining (\ref{eqn:deltadefn})) leads to our objective.

\paragraph{Objective for training with PMI.}
A similar objective {\bf PMI} can be obtained from (\ref{eqn:PMI}), 
by computing an approximate MLE, using the fact that the error between the empirical and true value of $\mbox{PMI}(w, w')$ is driven by the smaller term $p(w,w')$, and not the larger terms $p(w), p(w')$.

\begin{align*} 
\min_{\cbr{v_w}, C} & \sum_{w,w'} X_{w,w'} \rbr{\mbox{PMI}(w, w') - \inner{v_{w}}{v_{w'}}}^2 
\end{align*}
This is of course very analogous to classical VSM methods, with a novel reweighting method. 

Fitting to either of the objectives involves solving a version of {\em Weighted SVD} which is NP-hard, but empirically seems solvable in our setting via AdaGrad~\cite{duchi2011adaptive}.

\paragraph{Connection to GloVe.}
Compare  \textbf{SN} with the objective used by GloVe~\cite{pennington2014glove}:
\begin{equation*} 
\sum_{w,w'} f(X_{w,w'})( \log(X_{w,w'}) - \inner{v_w}{v_{w'}} - s_w - s_{w'} - C)^2
\end{equation*} 
with $f(X_{w,w'}) = \min\{ X_{w,w'}^{3/4}, 100\}.$
Their weighting methods and the need for {\em bias} terms $s_w, s_{w'},C$ were derived by trial and error; here
they are all predicted and given  meanings due to Theorem~\ref{thm:main}, specifically 
$s_w = \nbr{v_w}^2$. 


\paragraph{Connection to word2vec(CBOW).} 
The CBOW model in word2vec posits that the probability of a word $w_{k+1}$ as a function of the previous $k$ words $w_1, w_2, \ldots, w_k$:
$$
   p\rbr{ w_{k+1} \big|\cbr{w_i}_{i=1}^{k}} \propto \expinner{v_{w_{k+1}},  \frac{1}{k}\sum_{i=1}^{k} v_{w_i} }.
$$

This  expression seems mysterious since it depends upon the {\em average} word vector for the previous $k$ words. We show it  can be theoretically justified. Assume a simplified version of our model, where a small window of $k$ words is generated as follows: sample $c \sim \mathcal{C}$, where $\mathcal{C}$ is a uniformly random unit vector, then sample $\left(w_1, w_2, \dots, w_k\right) \sim \expinner{\sum^k_{i=1} v_{w_i}, c}/Z_c$. Furthermore, assume $Z_c = Z$ for any $c$.  
\begin{lem} In the simplified version of our model, the Maximum-a-Posteriori (MAP) estimate of $c$ given $\left(w_1, w_2, \dots, w_k\right)$ is $\frac{\sum_{i=1}^k v_{w_i}}{\|\sum_{i=1}^k v_{w_i} \|_2}$. 
\end{lem} 
\begin{proof} 
The $c$ maximizing $ p\rbr{ c | w_1, w_2, \dots, w_k } $ is the maximizer of $p(c) p\rbr{ w_1, w_2, \dots, w_k |c }$.
Since $p(c) = p(c')$ for any $c,c'$, and we have 
$
  p\rbr{ w_1, w_2, \dots, w_k | c } = \expinner{\sum_i v_{w_i}, c }/Z,
$ 
the maximizer is clearly $c = \frac{\sum_{i=1}^k v_{w_i}}{\|\sum_{i=1}^k v_{w_i} \|_2}$. 
\end{proof} 
Thus using the MAP estimate of $c_t$ gives  essentially the same expression as CBOW
apart from the rescaling, which is often omitted due to computational efficiency in empirical works.

\section{Explaining {\sc relations=lines}} 
\label{sec:interpretability}

As mentioned, word analogies like \textquotedblleft $a$:$b$::$c$:??\textquotedblright\ can be  
solved via a linear algebraic expression:
\begin{equation} 
\argmin_{d} \nbr{v_a - v_b - v_c + v_d}_2^2, \label{eqn:query}
\end{equation}
where vectors have been normalized such that $\nbr{v_d}_2=1$. 
This suggests that the semantic relationships being tested in the analogy are characterized by a straight line,\footnote{Note that this interpretation has been disputed; e.g., it is argued in~\newcite{levy2014linguistic} that (\ref{eqn:query}) can be understood using only the classical
connection between inner product and word similarity, using which the objective (\ref{eqn:query}) is slightly improved to a different objective called 3COSMUL. However, this \textquotedblleft explanation\textquotedblright\ is still dogged by the issue of large termwise error pinpointed here, since inner product is only a rough approximation to word similarity. Furthermore, the experiments in Section~\ref{sec:experiment} clearly support the {\sc relations=lines} interpretation.}  referred to earlier as {\sc relations=lines}.

Using our model we will show the following for low-dimensional embeddings: for each such relation $R$ there is a direction $\mu_R$ in space such that for any word pair $a, b$ satisfying the relation,  $v_a -v_b$  is like $\mu_R$ plus some noise vector. This happens for relations satisfying a certain condition described below.  Empirical results supporting this theory appear in Section~\ref{sec:experiment}, where this linear structure is further leveraged to slightly improve analogy solving. 

A side product of our argument will be a mathematical explanation of the empirically well-established superiority of 
low-dimensional word embeddings over high-dimensional ones in this setting~\cite{levy2014linguistic}. 
As mentioned earlier, the usual explanation that smaller models generalize better is fallacious.

We first sketch what was missing in prior attempts to prove versions of {\sc relations=lines} from first principles.  
The basic issue is approximation error: the difference between the best 
solution and the 2nd best solution to (\ref{eqn:query}) is typically small, 
whereas the approximation error in the objective in the low-dimensional solutions
is larger. For instance, if one uses our {\bf PMI} objective, then the weighted average of the termwise error in~(\ref{eqn:PMI}) is $17\%$, and the expression in (\ref{eqn:query}) above contains six inner products. Thus in principle the approximation error could lead to a failure of the method and the emergence of linear relationship, but it does not.

\paragraph{Prior explanations.}
\citet{pennington2014glove} try to propose a model 
where such linear relationships should occur {\em by design}. 
They posit that {\em queen} is a solution to the analogy
\textquotedblleft {\em man}:{\em woman}::{\em king}:??\textquotedblright\ because 
\begin{equation} \label{eqn:analogy}
\frac{p(\chi\mid king)}{p(\chi \mid queen)} \approx \frac{p(\chi\mid man)}{p(\chi \mid woman)},
\end{equation}
where $p(\chi\mid king)$ denotes the conditional probability of seeing word $\chi$ in a small window
of text around ${king}$. Relationship (\ref{eqn:analogy}) is intuitive since  both sides  will be $\approx 1$ for gender-neutral $\chi$ like ``{\em walks}'' or ``{\em food}'',  will be $>1$ when $\chi$ is like ``{\em he, Henry}'' and will be $<1$ when $\chi$ is like ``{\em dress, she, Elizabeth}.'' This was also observed by~\citet{levy2014linguistic}.
 Given (\ref{eqn:analogy}), they then posit that the correct model describing word embeddings in terms of word occurrences must be a {\em homomorphism} from $({\Re^d, +})$ to $(\Re^+, \times)$,
 so vector differences map to ratios of probabilities.  This leads to the expression
$$ p_{w,w'} = \inner{v_{w}}{v_{w'}} + b_w + b_{w'},$$
and their method is a (weighted) least squares fit for this expression. 
One shortcoming of this argument is that the homomorphism assumption  {\em assumes} the linear relationships instead of explaining them from a more basic principle. More importantly, 
the empirical fit to the homomorphism has nontrivial approximation error, high enough that  it does not 
imply the desired strong linear relationships.

\citet{levy2014neural} show that empirically, skip-gram vectors satisfy 
\begin{equation}  \inner{v_{w}}{ v_{w'} }\approx \mbox{PMI}(w, w') \label{eqn:pmi2}
\end{equation}
up to some shift.
They also give an argument suggesting this relationship must be present  if the solution is  allowed to be very high-dimensional. Unfortunately, that  argument does not 
extend to low-dimensional embeddings. 
Even if it did, the issue of termwise approximation error remains.
%

\paragraph{Our explanation.}
The current paper has introduced a generative model to theoretically explain  the emergence of relationship (\ref{eqn:pmi2}). However, as noted after Theorem~\ref{thm:main}, the issue of high approximation error does not go away either in theory or in the empirical fit.
We now  show that the isotropy of word vectors  (assumed in the theoretical model and verified empirically) implies
that even a weak version of~(\ref{eqn:pmi2}) is enough to imply the emergence of the observed linear relationships in low-dimensional embeddings.


This argument will assume the analogy in question involves a relation that obeys Pennington et al.'s suggestion in (\ref{eqn:analogy}).
 Namely, for such a relation $R$ there exists  function
$\nu_R(\cdot)$ depending only upon $R$ such that for any $a, b$ satisfying $R$ there is a {\em noise function} $\xi_{a,b,R}(\cdot )$  for which:
	\begin{equation}
	\frac{p(\chi\mid a)}{p(\chi \mid b)} = \nu_R(\chi)\cdot \xi_{a,b,R}(\chi) \label{eqn:relation-mutliplicative}
	\end{equation}
For different words $\chi$ there is huge variation in (\ref{eqn:relation-mutliplicative}), so the multiplicative noise may be large.

Our goal is to show that the low-dimensional word embeddings have the property that there is a vector $\mu_R$ such that for every pair of words $a, b$ in that relation, $v_{a} - v_{b} = \mu_R + \mbox{noise vector}$, where the noise vector is small.

Taking logarithms of~\eqref{eqn:relation-mutliplicative} results in:
\begin{equation}
	\log\left(\frac{p(\chi\mid a)}{p(\chi \mid b)}\right) = \log(\nu_R(\chi) )+ \zeta_{a,b,R}(\chi)\label{eqn:relationship-log}
\end{equation}

Theorem~\ref{thm:main} implies that the left-hand side simplifies to $\log\left(\frac{p(\chi\mid a)}{p(\chi \mid b)}\right) = \frac{1}{d}\inner{v_{\chi}}{v_a-v_b} +\epsilon_{a,b}(\chi)$ where  $\epsilon$ captures the small approximation errors induced by the inexactness of Theorem~\ref{thm:main}. This adds yet more noise! Denoting by $V$ the $n\times d$  matrix whose rows are the $v_{\chi}$ vectors, we rewrite~\eqref{eqn:relationship-log} as:
\begin{equation}
V(v_a-v_b) = d\log(\nu_R) + \zeta'_{a,b,R}\label{eqn:linear-regression}
\end{equation}
where $\log(\nu_R)$ in the element-wise log of vector $\nu_R$ and $\zeta'_{a,b,R} = d(\zeta_{a,b,R} - \epsilon_{a,b,R})$ is the noise. 

In essence,~\eqref{eqn:linear-regression} shows that $v_a -v_b$ is a solution to a linear regression in $d$ variables and $m$ constraints, with $\zeta'_{a, b, R}$ being the \textquotedblleft noise.\textquotedblright\  
 The \emph{design matrix} in the regression is $V$, the matrix of all word vectors, which in our model
(as well as empirically) satisfies an isotropy condition. This  makes
it random-like, and thus solving the regression by left-multiplying by $V^{\dagger}$, the pseudo-inverse of $V$,   
ought to \textquotedblleft denoise\textquotedblright\
effectively. We now show that it does. 

Our model assumed the set of all word vectors satisfies bulk properties similar to a 
set of Gaussian vectors.  
The next theorem will only need the following weaker properties. 
(1) The smallest non-zero singular value of $V$ is larger than some constant $c_1$ times the quadratic mean of the singular values, namely,  $\|V\|_F/\sqrt{d}$. 
Empirically we find $c_1 \approx 1/3$ holds; see Section~\ref{sec:experiment}.
(2) The left singular vectors behave like random vectors with respect to $\zeta'_{a,b,R}$, namely, have inner product at most $c_2 \|\zeta'_{a,b,R}\|/\sqrt{n}$ with  $\zeta'_{a,b,R}$, for some constant $c_2$. 
(3) The max norm of a row in $V$ is $O(\sqrt{d})$.
The proof is included in the appendix. 
 


\begin{thm}[Noise reduction]
\label{thm:noise-reduction}
Under the conditions of the previous paragraph, the noise in the dimension-reduced semantic vector space satisfies 
\begin{equation*}
\|\bar{\zeta}_{a,b,R}\|_2 \lesssim \|\zeta'_{a,b,R}\|_2\frac{\sqrt{d}}{n}.
\end{equation*} 
As a corollary, the relative error in the dimension-reduced space is a factor of $\sqrt{d/n}$ smaller. 
\end{thm}

\section{Experimental verification}\label{sec:experiment} 

In this section, we provide experiments empirically supporting our generative model. 

\paragraph{Corpus.}
All word embedding vectors are trained on the English Wikipedia (March 2015 dump). It is pre-processed by standard approach (removing non-textual elements, sentence splitting, and tokenization), leaving about $3$ billion tokens. 
Words that appeared less than $1000$ times in the corpus are ignored, resulting in a vocabulary of $68,430$.
The co-occurrence is then computed using windows of 10 tokens to each side of the focus word.

\paragraph{Training method.}
Our embedding vectors are trained by optimizing  the \textbf{SN} objective 
using AdaGrad~\cite{duchi2011adaptive} with initial learning rate of $0.05$ and 100 iterations.  The \textbf{PMI} objective derived from (\ref{eqn:PMI}) was also used. \textbf{SN} has average (weighted) term-wise error of $5$\%, and \textbf{PMI} has $17$\%.  We observed that \textbf{SN} vectors typically fit the model better and have better performance, which can be explained by larger errors in PMI, as implied by Theorem~\ref{thm:main}.  So, we only report the results for \textbf{SN}. 

For comparison, GloVe and two variants of word2vec (skip-gram and CBOW) vectors are trained. 
GloVe's vectors are trained on the same co-occurrence as \textbf{SN} with the default parameter values.\footnote{\url{http://nlp.stanford.edu/projects/glove/}} word2vec vectors are trained using a window size of 10, with other parameters set to default values.\footnote{\url{https://code.google.com/p/word2vec/}}

\newcommand{\scalefac}{0.18}
\newcommand{\hmargin}{-.1in}
\newcommand{\vmargin}{-.0in}

\begin{figure*}
	\centering
	\hspace{\hmargin}
	\subfigure[\textbf{SN}]{		\includegraphics[height=\scalefac\textwidth]{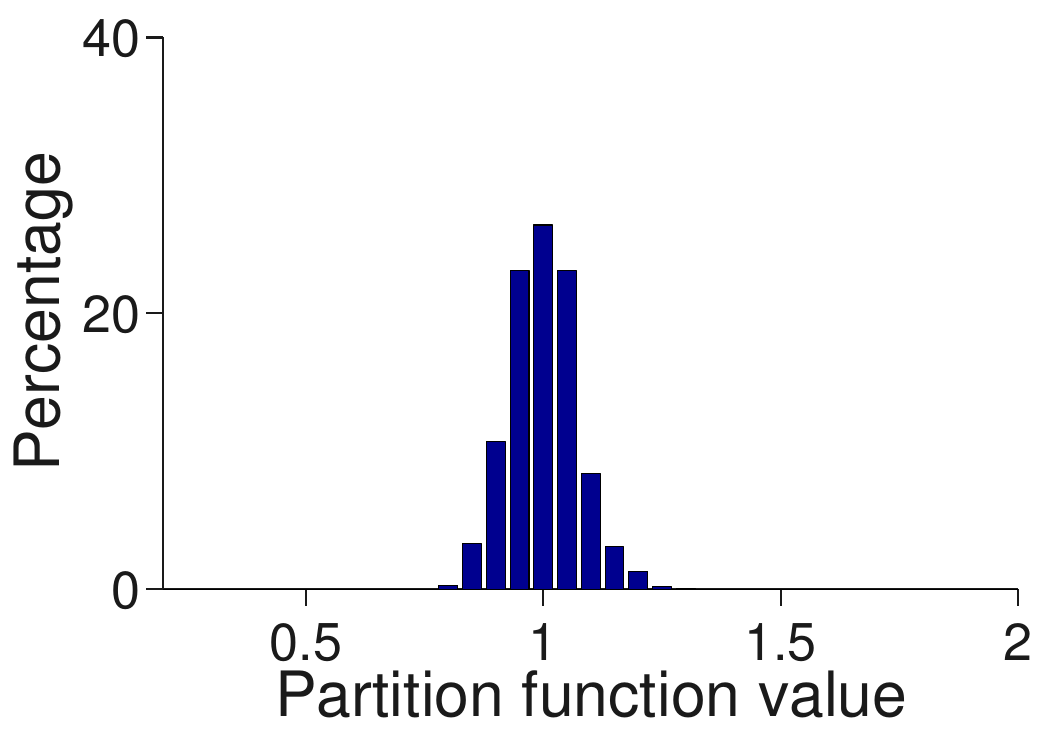}}
	\hspace{\hmargin}
	\subfigure[GloVe]{		\includegraphics[height=\scalefac\textwidth]{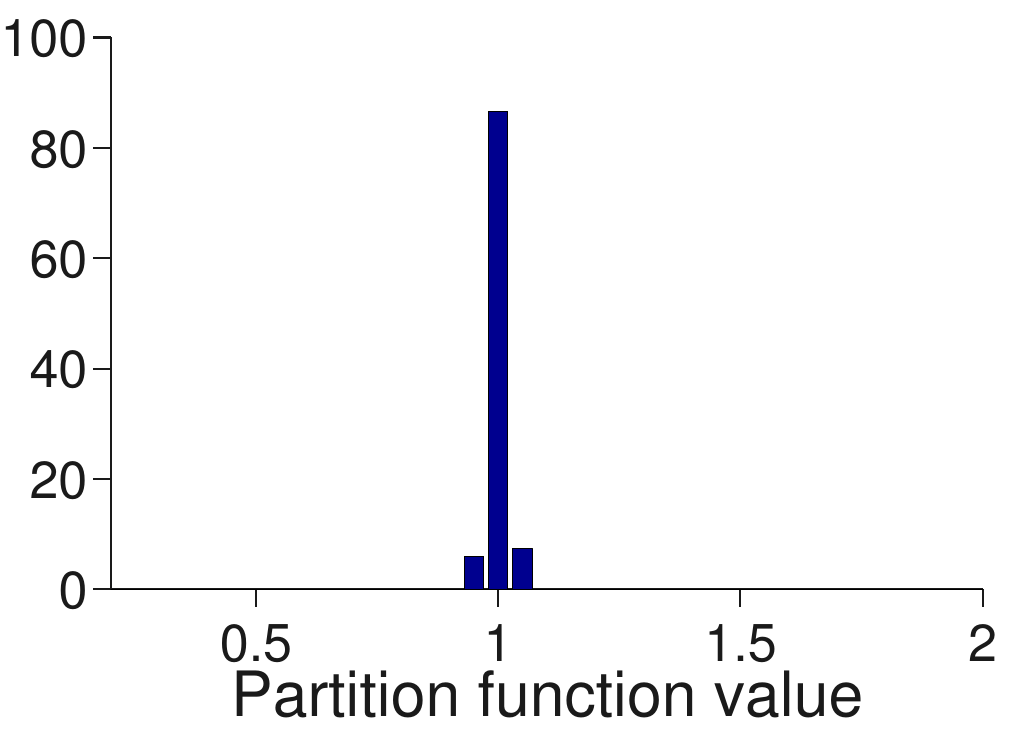}}
	\hspace{\hmargin}
	\subfigure[CBOW]{	\includegraphics[height=\scalefac\textwidth]{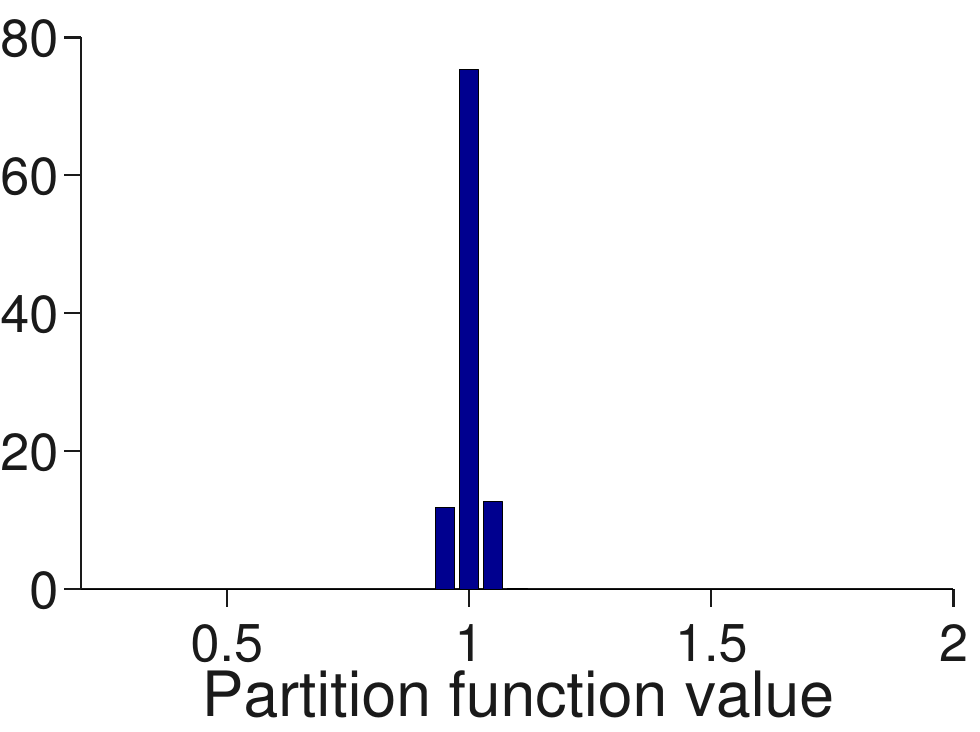}}
	\hspace{\hmargin}
	\subfigure[skip-gram]{		\includegraphics[height=\scalefac\textwidth]{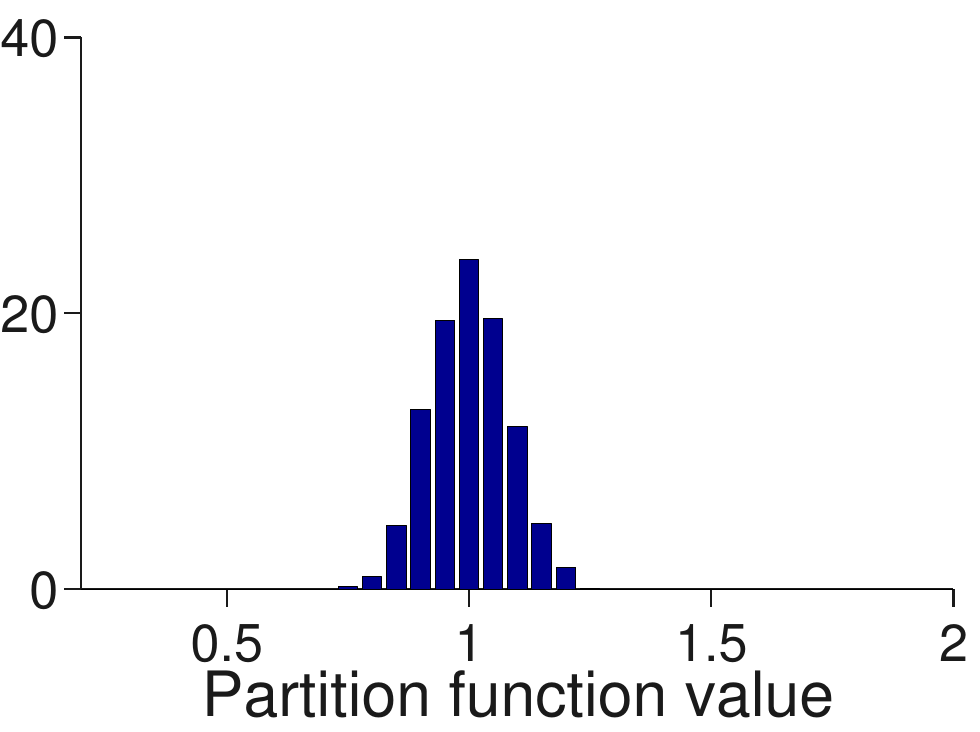}}
	\caption{The partition function $Z_c$. The figure shows the histogram of $Z_c$ for $1000$ random vectors $c$ of appropriate norm, as defined in the text. The $x$-axis is normalized by the mean of the values. The values $Z_c$ for different $c$ concentrate around the mean, mostly in $[0.9, 1.1]$. This concentration phenomenon is predicted by our analysis.
	\label{fig:zc}}
	\vspace{-.1in}
\end{figure*} 

\subsection{Model verification}
\label{subs: mverif} 
Experiments were run to test our modeling assumptions. First, we tested two counter-intuitive properties: the concentration of the partition function $Z_c$ for different discourse vectors $c$ (see Theorem~\ref{thm:Z_c}), and the random-like behavior of the matrix of word embeddings in terms of its singular values   (see Theorem~\ref{thm:noise-reduction}).
 For comparison we also 
 tested these properties for word2vec and GloVe vectors, though they are trained by different objectives. 
 Finally, we tested the linear relation between the squared norms of our word vectors and the logarithm of the word frequencies, as implied by Theorem~\ref{thm:main}.

\begin{figure}[!t]
	\centering
	\includegraphics[height=0.3\columnwidth]{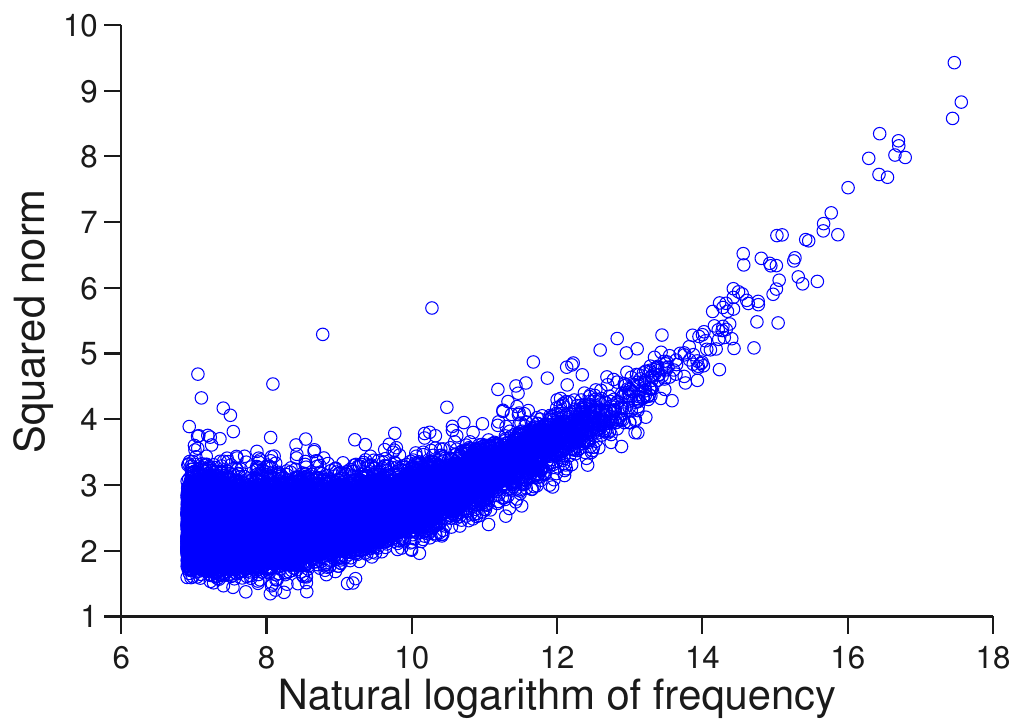}
	\caption{The linear relationship between the squared norms of our word vectors and the logarithms of the word frequencies.
		Each dot in the plot corresponds to a word, where $x$-axis is the natural logarithm of the word frequency, and $y$-axis is the squared norm of the word vector. The Pearson correlation coefficient between the two is 0.75, indicating a significant linear relationship, which strongly supports our mathematical prediction, that is, equation~\eqref{eqn:singleword} of Theorem~\ref{thm:main}. 
	\label{fig:normVSfreq}
	}
	\vspace{-2mm}
\end{figure}

\paragraph{Partition function.} Our theory predicts the counter-intuitive concentration of the partition function  $Z_{c} = \sum_{w'} \exp(c^\top v_{w'})$ for a random discourse vector $c$ (see Lemma~\ref{thm:Z_c}). This is verified empirically by picking a uniformly random direction, of norm $\|c\|  = 4/\mu_w$, where $\mu_w$ is the average norm of the word vectors.\footnote{Note that our model uses the inner products between the discourse vectors and word vectors, so it is invariant  if the discourse vectors are scaled by $s$ while the word vectors are scaled by $1/s$ for any $s>0$. Therefore, one needs to choose the norm of $c$ properly. We assume $\|c\| \mu_w = \sqrt{d}/\kappa \approx 4$ for a constant $\kappa = 5$ so that it gives a reasonable fit to the predicted dynamic range of word frequencies according to our theory; see model details in Section~\ref{sec:model}.}
 Figure~\ref{fig:zc}(a) shows the histogram of $Z_c$ for $1000$ such randomly chosen $c$'s for our vectors. The values are concentrated, mostly in the range $[0.9,1.1]$ times the mean.  Concentration is also observed for other types of vectors, especially for GloVe and CBOW.

\paragraph{Isotropy with respect to singular values.} Our theoretical explanation of {\sc relations=lines} assumes that the matrix of word vectors behaves like a random matrix with respect to the properties of singular values.  In our embeddings, the quadratic mean of the singular values is 34.3, while the minimum non-zero singular value of our word vectors is 11. Therefore, the ratio between them is a small constant, consistent with our model. 
The ratios for GloVe, CBOW, and skip-gram are 1.4, 10.1, and 3.1, respectively, which are also small constants.

\paragraph{Squared norms v.s. word frequencies.}  
Figure~\ref{fig:normVSfreq} shows a scatter plot for the squared norms of our vectors and the logarithms of the word frequencies. A linear relationship is observed (Pearson correlation 0.75), thus supporting Theorem~\ref{thm:main}. The correlation is stronger  for high frequency words, possibly because the corresponding terms have higher weights in the training objective.

This correlation  
 is much weaker for other types of word embeddings. This is possibly because they have more free parameters (\textquotedblleft knobs to turn\textquotedblright), which imbue the embeddings with other properties.
 This can also cause the difference in the concentration of the partition function for the two methods.

\subsection{Performance on analogy tasks} 
We compare the performance of our word vectors on analogy tasks, specifically the 
two  testbeds GOOGLE and MSR~\cite{mikolov2013efficient,mikolov2013linguistic}.
The former contains $7874$ semantic questions such as ``\emph{man}:\emph{woman}::\emph{king}:??", and $10167$ syntactic ones such as ``\emph{run}:\emph{runs}::\emph{walk}:??."  The latter has $8000$ syntactic questions for adjectives, nouns, and verbs. 

To solve these tasks, we use linear algebraic queries.\footnote{One can instead use the 3COSMUL in~\cite{levy2014linguistic}, which increases the accuracy by about $3\%$. But it is not linear while our focus here is the linear algebraic structure.}
That is, first normalize the vectors to unit norm and then solve ``\emph{a}:\emph{b}::\emph{c}:??''  by 
\vspace{-.1in}
\begin{align}\label{eqn:exp_linearquery}
\argmin_{d} \nbr{v_a - v_b - v_c + v_d}_2^2.
\vspace{-.2in}
\end{align}
The algorithm succeeds if the best $d$ happens to be correct.
\begin{table}
	\centering
		\begin{tabular}{c | l | c c c c c c }
		  \hline
			& Relations                             & \textbf{SN}  & GloVe  & CBOW  & skip-gram   \\
			\hline\hline
			\multirow{3}{*}{G}
			& semantic                             &       0.84   &  0.85    &  0.79     & 0.73           \\
			& syntactic                              &       0.61   &    0.65    &  0.71    & 0.68            \\
			& total                                     &        0.71     &    0.73   & 0.74   & 0.70        \\
			\hline\hline
			\multirow{4}{*}{M}
			&    adjective        &     0.50 &  0.56  & 0.58 &    0.58     \\
			& noun                  &      0.69 &  0.70 &  0.56 & 0.58      \\
			& verb                   &      0.48   &    0.53 &   0.64    &0.56 \\
			&  total                    & 0.53   &  0.57       &    0.62    & 0.57   \\
			\hline
		\end{tabular}
	\caption{The accuracy on two word analogy task testbeds: G (the GOOGLE testbed); M (the MSR testbed). Performance is close to the state of the art despite using a generative model with provable properties. \vspace{-.15in} } 
	\label{tab:acc}	
\end{table}

The performance of different methods is presented in Table~\ref{tab:acc}.  Our vectors achieve performance comparable to the state of art on semantic analogies (similar accuracy as GloVe, better than word2vec).  On syntactic tasks, they achieve accuracy $0.04$ lower than GloVe and skip-gram, while CBOW typically outperforms the others.\footnote{It was earlier reported that skip-gram outperforms CBOW~\cite{mikolov2013efficient,pennington2014glove}. This may be due to the different training data sets and hyperparameters used.}
The reason is probably that our model ignores local word order, whereas the other models capture it to some extent. For example, a word ``\emph{she}'' can affect the context by a lot and determine if the next word is ``\emph{thinks}'' rather than ``\emph{think}''. 
Incorporating  such linguistic features in the model is left for future work.

%
 
\subsection{Verifying {\sc relations=lines}}
The theory in Section~\ref{sec:interpretability} predicts the existence of a direction for a relation, whereas  earlier~\citet{levy2014linguistic} had questioned if this phenomenon is real.  The experiment uses the analogy testbed, where each relation is tested using $20$ or more analogies. For each relation, we take the set of vectors $v_{ab} = v_a -v_b$ where the word pair $(a, b)$ satisfies the relation. Then calculate the top singular vectors of the matrix formed by these $v_{ab}$'s, and compute the cosine similarity (i.e., normalized inner product) of individual $v_{ab}$ to  the singular vectors. 
We observed that most $(v_a -v_b)$'s are correlated with the first singular vector, but have inner products around 0 with the second singular vector.  Over all relations, the average projection on the first singular vector is 0.51  (semantic: 0.58; syntactic: 0.46), and the average on the second singular vector is 0.035.
For example, Table~\ref{tab:verifyrd} shows the mean similarities and standard deviations on the first and second singular vectors for 4 relations. 
Similar results are also obtained for word embedings by GloVe and word2vec.
Therefore, the first singular vector can be taken as the direction associated with this relation, while the other components are like random noise, in line with our model. 

\begin{table*}
\vspace{5mm}
\centering
{
\begin{tabular}{c|ccccccc}
\hline 
relation & 1 & 2 & 3 & 4 & 5 & 6 & 7 \\
\hline\hline
1st & 0.65 $\pm$ 0.07 & 0.61 $\pm$ 0.09 & 0.52 $\pm$ 0.08 & 0.54 $\pm$ 0.18 & 0.60 $\pm$ 0.21 & 0.35 $\pm$ 0.17 & 0.42 $\pm$ 0.16 \\
2nd & 0.02 $\pm$ 0.28 & 0.00 $\pm$ 0.23 & 0.05 $\pm$ 0.30 & 0.06 $\pm$ 0.27 & 0.01 $\pm$ 0.24 & 0.07 $\pm$ 0.24 & 0.01 $\pm$ 0.25 \\
\hline
\end{tabular}
\begin{tabular}{c|ccccccc}
\hline
relation & 8 & 9 & 10 & 11 & 12 & 13 & 14 \\
\hline\hline
1st & 0.56 $\pm$ 0.09 & 0.53 $\pm$ 0.08 & 0.37 $\pm$ 0.11 & 0.72 $\pm$ 0.10 & 0.37 $\pm$ 0.14 & 0.40 $\pm$ 0.19 & 0.43 $\pm$ 0.14 \\
2nd & 0.00 $\pm$ 0.22 & 0.01 $\pm$ 0.26 & 0.02 $\pm$ 0.20 & 0.01 $\pm$ 0.24 & 0.07 $\pm$ 0.26 & 0.07 $\pm$ 0.23 & 0.09 $\pm$ 0.23 \\ 
\hline
\end{tabular}
}
\caption{The verification of relation directions on 2 semantic and 2 syntactic relations in the GOOGLE testbed. Relations include cap-com: capital-common-countries; cap-wor: capital-world; adj-adv: gram1-adjective-to-adverb; opp: gram2-opposite. For each relation, take $v_{ab} = v_a -v_b$ for pairs $(a, b)$ in the relation, and then calculate the top singular vectors of the matrix formed by these $v_{ab}$'s.  The row with label ``1st''/``2nd'' shows the cosine similarities of individual $v_{ab}$ to  the 1st/2nd singular vector (the mean and standard deviation). 
}
\label{tab:verifyrd}
\end{table*}

\begin{table}
\centering
\begin{tabular}{c|cccc}
\hline 
& \textbf{SN} & GloVe & CBOW & skip-gram \\
\hline\hline
w/o \textbf{RD} & 0.71 & 0.73 & 0.74 & 0.70\\
\hline
\textbf{RD}($k$ = 20) & 0.74 & 0.77 & 0.79 & 0.75 \\
\textbf{RD}($k$ = 30) & 0.79 & 0.80 & 0.82 & 0.80 \\
\textbf{RD}($k$ = 40) & 0.76 & 0.80 & 0.80 & 0.77 \\
\hline
\end{tabular}
\caption{The accuracy of the \textbf{RD} algorithm (i.e., the cheater method) on the GOOGLE testbed. 
The \textbf{RD} algorithm is described in the text.
For comparison, the row ``w/o \textbf{RD}'' shows the accuracy of the old method without using \textbf{RD}.}
\label{tab:rdperf}
\end{table}

\paragraph{Cheating solver for analogy testbeds.} The above linear structure suggests a better (but cheating) way to solve the analogy task. This uses the fact that the same semantic relationship (e.g., masculine-feminine,
singular-plural) is tested many times in the testbed. If a relation $R$ is represented by a direction $\mu_R$ then
the cheating algorithm can learn this direction (via rank 1 SVD) after seeing a few examples of the relationship. Then use the following method of solving ``\emph{a}:\emph{b}::\emph{c}:??'': look for a word $d$ such that $v_c -v_d$ has the largest projection on $\mu_R$, the relation direction for $(a,b)$. This can boost success rates by about $10\%$. 

The testbed can try to combat such cheating by giving analogy questions in a random order. But the cheating  algorithm can just {\em cluster}
the presented analogies to learn which of them are in the same relation. 
Thus the final algorithm, named analogy solver with relation direction ($\textbf{RD}$), is: 
take all vectors  $v_a - v_b$ for all the word pairs $(a,b)$ presented among the analogy questions and do $k$-means clustering on them; for each $(a,b)$, estimate the relation direction by taking the first singular vector of its cluster, and substitute that for $v_a -v_b$ in \eqref{eqn:exp_linearquery} when solving the analogy. 
Table~\ref{tab:rdperf} shows the performance on GOOGLE with different values of $k$; e.g.
using our \textbf{SN} vectors and $k=30$ leads to $0.79$ accuracy.
Thus future designers of analogy testbeds should remember not to test the same relationship too many times!
This still leaves other ways to cheat, such as learning the directions for  interesting semantic relations
from other collections of analogies. 

\paragraph{Non-cheating solver for analogy testbeds.} Now we show that even if a relationship
is tested only once in the testbed, there is a  way to use the above structure. 
 Given ``\emph{a}:\emph{b}::\emph{c}:??,'' the solver first finds the top $300$ nearest neighbors of $a$ and those of $b$, and then finds among these neighbors the top $k$ pairs $(a',b')$ so that the cosine similarities between $v_{a'} - v_{b'}$ and $v_a - v_b$ are largest. Finally, the solver uses these pairs to estimate the relation direction (via rank 1 SVD), and substitute this (corrected)
 estimate for $v_a -v_b$ in \eqref{eqn:exp_linearquery} when solving the analogy. 
This algorithm is named analogy solver with relation direction by nearest neighbors ($\textbf{RD-nn}$). 
Table~\ref{tab:rdnnperf} shows its performance, which consistently improves over the old method by about $3\%$.

\begin{table}
\centering
\begin{tabular}{c|cccc}
\hline 
& \textbf{SN} & GloVe & CBOW & skip-gram \\
\hline\hline
               w/o \textbf{RD-nn} & 0.71 & 0.73 & 0.74 & 0.70\\
\hline
\textbf{RD-nn} ($k$ = 10) & 0.71 & 0.74 & 0.77 & 0.73 \\
\textbf{RD-nn} ($k$ = 20) & 0.72 & 0.75 & 0.77 & 0.74 \\
\textbf{RD-nn} ($k$ = 30) & 0.73 & 0.76 & 0.78 & 0.74 \\
\hline
\end{tabular}
\caption{The accuracy of the \textbf{RD-nn} algorithm on the GOOGLE testbed. 
The algorithm is described in the text.
For comparison, the row ``w/o \textbf{RD-nn}'' shows the accuracy of the old method without using \textbf{RD-nn}.}
\label{tab:rdnnperf}
\end{table}
\section{Conclusions}
\label{sec:conclusions}
\vspace{-.05in}
A simple generative model has been introduced to explain
the classical PMI based word embedding models, as well as recent variants involving
energy-based models and matrix factorization. The model yields an optimization objective with essentially \textquotedblleft no knobs to turn\textquotedblright, yet the embeddings lead to good 
performance on analogy tasks, and fit other predictions of our  generative model. A model with fewer knobs to turn should be seen as a better scientific explanation ({\em Occam's razor}), and certainly makes the embeddings more interpretable. 

The  spatial isotropy of word vectors is both an assumption in our model, and also a
new empirical finding of our paper. We feel it may help with further development of language models.  It is important for explaining the success of solving analogies via low dimensional 
vectors ({\sc relations=lines}). It also implies that semantic relationships among words manifest themselves as special directions among word embeddings (Section~\ref{sec:interpretability}), which lead to a cheater algorithm for solving analogy testbeds.


Our model is tailored to capturing semantic similarity, more akin to a log-linear dynamic topic model. In particular, local word order is unimportant. Designing similar generative models (with provable and
interpretable properties) with linguistic features is left for future work. 

\section*{Acknowledgements}

We thank the editors of TACL for granting a special relaxation of the page limit for our paper.  We thank Yann LeCun, Christopher D. Manning, and Sham Kakade for helpful
discussions at various stages of this work. 

This work was supported in part by NSF grants 
CCF-1527371, DMS-1317308, Simons Investigator Award, Simons Collaboration Grant,
and ONR-N00014-16-1-2329. Tengyu Ma was supported in addition by Simons Award in Theoretical Computer Science and IBM PhD Fellowship.

\bibliographystyle{plainnat}
\bibliography{semantic_vector}

\newpage
\appendix

\section{Proofs of Theorem 1} \label{sec:proofmain}

In this section we prove Theorem~\ref{thm:main} and Lemma~\ref{thm:Z_c} (restated below). 

\begin{oneshot}{Theorem~\ref{thm:main}}
	Suppose the word vectors satisfy equation~\eqref{eqn:lem_zc}, and window size $q = 2$. Then, 
	\begin{align}
	\log p(w,w')   & =  \frac{\|v_w + v_{w'}\|_2^2 }{2d}- 2\log Z  \pm \epsilon,  \label{eqn:twowords_cooc_app} \\
	\log p(w) & =  \frac{\nbr{v_{w}}_2^2}{2d}  - \log Z \pm \epsilon. \label{eqn:singleword_app}
	\end{align}
	for $\epsilon = O(\epsilon_z) +\widetilde{O}(1/d) + O(\epsilon_2)$. 
	Jointly these imply: 
	\begin{equation}
	\mbox{PMI}~(w, w') = \frac{\langle v_w, v_{w'}\rangle}{d}  \pm O(\epsilon). \label{eqn:PMI_app}
	\end{equation}
\end{oneshot}

\begin{oneshot}{Lemma~\ref{thm:Z_c}}
		If the word vectors satisfy the bayesian prior $v =  s \cdot \hat{v}$, where $\hat{v}$ is from the spherical Gaussian distribution, and $s$ is a scalar random variable, then with high probability the entire
		ensemble of word vectors satisfies that
		\begin{equation}
		\Pr_{c\sim \mathcal{C}}\left[(1-\epsilon_z)Z\le Z_c\le (1+\epsilon_z) Z\right]\ge 1-\delta \label{eqn:lem_zc_app},
		\end{equation}
		%
		%
		for $\epsilon_z = \widetilde{O}(1/\sqrt{n})$, and $\delta = \exp(-\Omega(\log^2 n)) $.
\end{oneshot}


We first prove Theorem~\ref{thm:main} using Lemma~\ref{thm:Z_c}, and Lemma~\ref{thm:Z_c} will be proved in Section~\ref{sec:zc}. Please see Section~\ref{sec:model} of the main paper for the intuition of the proof and a cleaner sketch without too many technicalities. 
\begin{proof}[Proof of Theorem~\ref{thm:main}]

Let $c$ be the hidden discourse that determines the probability of word $w$, and $c'$ be the next one that determines $w'$. We use $p(c'\vert c)$ to denote the Markov kernel (transition matrix) of the Markov chain.  Let $\mathcal{C}$ be the stationary distribution of discourse vector $c$, and $\mathcal{D}$ be the joint distribution of $(c,c')$. 
We marginalize over the contexts $c,c'$ and then use the independence of $w,w'$ conditioned on $c,c'$,

\begin{equation}
p(w,w') = \Exp_{(c,c')\sim \mathcal{D}}\left[\frac{\expinner{v_w,c}}{Z_c}\frac{\expinner{v_{w'},c'}}{Z_{c'}}\right]\label{eqn:eqn20}\end{equation}

We first get rid of the partition function $Z_c$ using Lemma~\ref{thm:Z_c}. As sketched in the main paper, essentially we will replace $Z_c$ by $Z$ in equation~\eqref{eqn:eqn20}, though a very careful control of the approximation error is required. Formally, 
Let $\mathcal{F}_1$ be the event that $c$ satisfies
\begin{equation}
	(1-\epsilon_z)Z\le Z_c\le (1+\epsilon_z)Z\,.\label{eqn:inter10}
\end{equation} Similarly, let $\mathcal{F}_2$ be the even that $c'$ satisfies	$(1-\epsilon_z)Z\le Z_{c'}\le (1+\epsilon_z)Z$, and let $\mathcal{F} = \mathcal{F}_1\cap \mathcal{F}_2$, and $\overline{\mathcal{F}}$ be its negation. Moreover, let $\mathbf{1}_{\mathcal{F}}$ be the indicator function for the event $\mathcal{F}$. Therefore by Lemma~\ref{thm:Z_c} and union bound, we have $\Exp[\mathbf{1}_{\mathcal{F}}] = \Pr[\mathcal{F}] \ge 1-\exp(-\Omega(\log^2 n))$.  

We first decompose the integral (\ref{eqn:eqn20}) into the two parts according to whether event $\mathcal{F}$ happens, 

\begin{align}
p(w,w') \nonumber
&=  \Exp_{(c,c')\sim \mathcal{D}}\left[ \frac{1}{Z_cZ_{c'}}\expinner{v_w,c}\expinner{v_{w'},c'} \mathbf{1}_{\mathcal{F}}\right]
\nonumber\\
& + \Exp_{(c,c')\sim \mathcal{D}}\left[\frac{1}{Z_cZ_{c'}}\expinner{v_w,c}\expinner{v_{w'},c'}\mathbf{1}_{\overline{\mathcal{F}}}\right]
\label{eqn:inter3}
\end{align}
We bound the first quantity on the right hand side using (\ref{eqn:lem_zc}) and the definition of $\mathcal{F}$. 
\newcommand{\superexp}[1]{\Exp_{(c,c')\sim \mathcal{D}}\left[#1\right]}
\newcommand{\littleexp}[1]{\Exp\left[#1\right]}
\newcommand{\outexpc}[1]{\Exp_{c}\left[#1\right]}
\newcommand{\outexpcp}[1]{\Exp_{c'}\left[#1\right]} 
\newcommand{\innerexpc}[1]{\Exp_{c' \mid c}\left[#1\right]} 
\newcommand{\innerexpcp}[1]{\Exp_{c \mid c'}\left[#1\right]}
\begin{align}
	& \Exp_{(c,c')\sim \mathcal{D}}\left[\frac{1}{Z_cZ_{c'}}\expinner{v_w,c}\expinner{v_{w'},c'} \mathbf{1}_{\mathcal{F}}\right]\nonumber \\
	& \le (1+\epsilon_z)^2 \frac{1}{Z^2} \superexp{\expinner{v_w,c}\expinner{v_{w'},c'} \mathbf{1}_{\mathcal{F}}}\label{eqn:inter1}
\end{align}
For the second quantity of the right hand side of~\eqref{eqn:inter3}, we have by Cauchy-Schwartz, 
\begin{align}
	&\left(\superexp{\frac{1}{Z_cZ_{c'}}\expinner{v_w,c}\expinner{v_{w'},c'} \mathbf{1}_{\overline{\mathcal{F}}}}\right)^2\nonumber\\
	& \le \left(\superexp{\frac{1}{Z_c^2}\expinner{v_w,c}^2\mathbf{1}_{\overline{\mathcal{F}}}}\right)\left(\superexp{\int_{c,c'} \frac{1}{Z_{c'}^2}\expinner{v_{w'},c'}^2\mathbf{1}_{\overline{\mathcal{F}}} }\right)\nonumber\\
	& \le \left(\outexpc{\frac{1}{Z_c^2}\expinner{v_w,c}^2\innerexpc{\mathbf{1}_{\overline{\mathcal{F}}}}}\right)\left(\outexpcp{ \frac{1}{Z_{c'}^2}\expinner{v_{w'},c'}^2\innerexpcp{\mathbf{1}_{\overline{\mathcal{F}}}}}\right)\,.\label{eqn:eqn9}
\end{align}
%


Using the fact that $Z_c\ge 1$, then we have that 
\begin{equation*}
\outexpc{\frac{1}{Z_c^2}\expinner{v_w,c}^2\innerexpc{\mathbf{1}_{\overline{\mathcal{F}}}}}\le \outexpc{\expinner{v_w,c}^2\innerexpc{\mathbf{1}_{\overline{\mathcal{F}}}}} 
\end{equation*}

We can split that expectation as 
\begin{equation}
\Exp_c \left[\expinner{v_w,c}^2 \mathbf{1}_{\langle v_w,c \rangle > 0}\ \innerexpc{\mathbf{1}_{\overline{\mathcal{F}}}} \right] + \Exp_c \left[ \expinner{v_w,c}^2 \mathbf{1}_{\langle v_w,c \rangle < 0} \innerexpc{\mathbf{1}_{\overline{\mathcal{F}}}}\right]\,.\label{eqn:eqn21}
\end{equation}
The second term of (\ref{eqn:eqn21}) is upper bounded by 
$$\Exp_{c,c'} [\mathbf{1}_{\overline{\mathcal{F}}}]\le \exp(-\Omega(\log^2 n))$$ 
We proceed to the first term of (\ref{eqn:eqn21}) and observe the following property of it:
$$ \Exp_c \left[\expinner{v_w,c}^2 \mathbf{1}_{\langle v_w,c \rangle > 0}\ \innerexpc{\mathbf{1}_{\overline{\mathcal{F}}}} \right] \leq \Exp_c \left[\expinner{\alpha v_w,c}^2 \mathbf{1}_{\langle v_w,c \rangle > 0}\ \innerexpc{\mathbf{1}_{\overline{\mathcal{F}}}} \right] \leq \Exp_c \left[\expinner{\alpha v_w,c}^2 \innerexpc{\mathbf{1}_{\overline{\mathcal{F}}}} \right]$$
where $\alpha > 1$. Therefore, it's sufficient to bound
$$\Exp_c \left[\expinner{v_w,c}^2 \innerexpc{\mathbf{1}_{\overline{\mathcal{F}}}} \right] $$
when $\|v_w\| = \Omega(\sqrt{d})$.   

Let's denote by $z$ the random variable $2\inner{v_w}{c}$. 

Let's denote 
$r(z) = \Exp_{c' \mid z} [\mathbf{1}_{\overline{\mathcal{F}}}]$ 
which is a function of $z$ between $[0,1]$. We wish to upper bound $\Exp_c \left[\exp(z)r(z)\right]$. 
The worst-case $r(z)$ can be quantified using a continuous version of Abel's inequality as proven in Lemma~\ref{lem:abel}, which gives 
\begin{equation}
\Exp_c[\exp(z)r(z)]\le \Exp\left[ \exp(z)\indicator{[t,+\infty]}(z)\right]
\end{equation}
where $t$ satisfies that 
$\Exp_c[\mathbf{1}_{[t,+\infty]}] = \Pr[z\ge t] = \Exp_c[r(z)] \leq \exp(-\Omega(\log^2 n))$. 
Then, we claim $\Pr[z\ge t] \le \exp(-\Omega(\log^2 n))$ implies that $t \ge \Omega(\log^{.9} n)$.

If $c$ were distributed as $\mathcal{N}(0, \frac{1}{d} I)$, this would be a simple tail bound. However, as $c$ is distributed uniformly on the sphere, this requires special care, and the claim follows by applying Lemma \ref{l:sphericalangle} instead. 
%
%
Finally, applying Corollary \ref{c:abelspherical}, we have: 
\begin{equation}
\Exp[\exp(z)r(z)]\le \Exp\left[ \exp(z)\indicator{[t,+\infty]}(z)\right] = \exp(-\Omega(\log^{1.8}n))
\end{equation}
We have the same bound for $c'$ as well.  
%
Hence, for the second quantity of the right hand side of~\eqref{eqn:inter3}, we have
\begin{align}
&\superexp{\frac{1}{Z_cZ_{c'}}\expinner{v_w,c}\expinner{v_{w'},c'} \mathbf{1}_{\overline{\mathcal{F}}}}\nonumber\\
&\le \left(\outexpc{\frac{1}{Z_c^2}\expinner{v_w,c}^2\innerexpc{\mathbf{1}_{\overline{\mathcal{F}}}}}\right)^{1/2} \left(\outexpcp{ \frac{1}{Z_{c'}^2}\expinner{v_{w'},c'}^2\innerexpcp{\mathbf{1}_{\overline{\mathcal{F}}}}}\right)^{1/2} \nonumber\\ 
& \le \exp(-\Omega(\log^{1.8}n))
\label{eqn:rhsbound}
\end{align}
%
where the first inequality follows from Cauchy-Schwartz, and the second from the calculation above. 
Combining (\ref{eqn:inter3}), (\ref{eqn:inter1}) and (\ref{eqn:rhsbound}), we obtain

\begin{align*}
	p(w,w')
	&\le (1+\epsilon_z)^2 \frac{1}{Z^2} \superexp{\expinner{v_w,c}\expinner{v_{w'},c'} \mathbf{1}_{\overline{\mathcal{F}}}}+ \frac{1}{n^2}\exp(-\Omega(\log^{1.8}n))\nonumber\\
	&\le (1+\epsilon_z)^2 \frac{1}{Z^2} \superexp{\expinner{v_w,c}\expinner{v_{w'},c'}} + \delta_0\nonumber\\
\end{align*}


where $\delta_0 = \exp(-\Omega(\log^{1.8}n))Z^2 \le \exp(-\Omega(\log^{1.8}n))$ by the fact that $Z \le \exp(2\kappa)n = O(n)$. Note that $\kappa$ is treated as an absolute constant throughout the paper. 
On the other hand, we can lowerbound similarly
\begin{align*}
	p(w,w') 
	&\ge (1-\epsilon_z)^2 \frac{1}{Z^2} \superexp{\expinner{v_w,c}\expinner{v_{w'},c'} \mathbf{1}_{\overline{\mathcal{F}}}} \\
	&\ge (1-\epsilon_z)^2 \frac{1}{Z^2} \superexp{\expinner{v_w,c}\expinner{v_{w'},c'}}  -  \frac{1}{n^2}\exp(-\Omega(\log^{1.8}n))\\
	&\ge (1-\epsilon_z)^2 \frac{1}{Z^2} \superexp{\expinner{v_w,c}\expinner{v_{w'},c'}}  - \delta_0
\end{align*}
Taking logarithm, the multiplicative error translates to a additive error 
\begin{align*}
	\log p(w,w') 
	&= \log\left(  \superexp{\expinner{v_w,c}\expinner{v_{w'},c'}}  \pm \delta_0\right) - 2\log Z + 2\log(1\pm \epsilon_z) \\
\end{align*}
For the purpose of exploiting the fact that $c,c'$ should be close to each other, we further rewrite $\log p(w,w')$ by re-organizing the expectations above,
\begin{align}
	\log p(w,w') 
	&= \log\left( \Exp_c\left[\expinner{v_w,c}\Exp_{c' \mid c}[\expinner{v_{w'},c'}]\right]\pm \delta_0\right) - 2\log Z + 2 \log(1\pm \epsilon_z)\nonumber\\
	& = \log\left( \Exp_c\left[\expinner{v_w,c} A(c)\right]\pm \delta_0\right) - 2\log Z + 2 \log(1\pm \epsilon_z)\label{eqn:ac}
\end{align}
where the inner integral which is denoted by $A(c)$, 
$$A(c) := \Exp_{c' \mid c}\left[\expinner{v_{w'},c'}\right]$$
Since $\|v_w\| \le \kappa\sqrt{d}$.  
 Therefore we have that $\langle v_{w},c- c'\rangle \le \|v_w\|\|c-c'\|\le \kappa\sqrt{d}\|c-c'\|$. 
	
Then we can bound $A(c)$ by

\begin{align*}
A(c) &= \Exp_{c' \mid c} \left[\expinner{v_{w'},c'}\right]	 \\
& = \expinner{v_{w'},c}\Exp_{c' \mid c}\left[\expinner{v_{w'},c'-c}\right]\\
& \le \expinner{v_{w'},c}\Exp_{c'|c}[\exp(\kappa\sqrt{d}\|c-c'\|)]\\
&\le (1+\epsilon_2)\expinner{v_{w'},c}
\end{align*}
where the last inequality follows from our model assumptions. To derive a lower bound of $A(c)$, observe that 

$$\Exp_{c'|c}[\exp(\kappa\sqrt{d}\|c-c'\|)] + \Exp_{c'|c}[\exp(-\kappa\sqrt{d}\|c-c'\|)]\ge 2$$

Therefore, our model assumptions imply that 

$$ \Exp_{c'|c}[\exp(-\kappa\sqrt{d}\|c-c'\|)]\ge 1-\epsilon_2$$

Hence,  
\begin{align*}
	A(c) 
	&=\expinner{v_{w'},c}\Exp_{c' \mid c}\expinner{v_{w'},c'-c} \\
	& \ge \expinner{v_{w'},c}\Exp_{c' \mid c}\exp(-\kappa\sqrt{d}\|c-c'\|)\\
	&\ge (1-\epsilon_2)\expinner{v_{w'},c}
\end{align*}

Therefore, we obtain that $A(c) = (1\pm \epsilon_2)\expinner{v_{w'},c}$. Plugging the just obtained estimate of $A(c)$ into the equation ~\eqref{eqn:ac}, we get that 


\begin{align}
	\log p(w,w') 
	&= \log\left( \Exp_c\left[\expinner{v_w,c}A(c)\right] \pm \delta_0\right) - 2\log Z + 2 \log(1\pm \epsilon_z)\nonumber\\
	& = \log\left(\Exp_{c}\left[(1 \pm \epsilon_2) \expinner{v_w,c}\expinner{v_{w'},c}\right] \pm \delta_0\right) - 2\log Z + 2 \log(1\pm \epsilon_z)\nonumber\\
	& = \log\left(\Exp_{c}\left[\expinner{v_w+v_{w'},c}\right]\pm \delta_0\right) - 2\log Z + 2\log(1\pm \epsilon_z) + \log(1\pm \epsilon_2)\label{eqn:eqn19}
\end{align}

Now it suffices to compute $ \Exp_{c}[\expinner{v_w+v_{w'},c}]$. Note that if $c$ had the distribution $\mathcal{N}(0,\frac{1}{d}I)$, which is very similar to uniform distribution over the sphere,  then we could get straightforwardly 
$ \Exp_{c}[\expinner{v_w+v_{w'},c}] = \exp(\|v_w+v_{w'}\|^2/(2d))$.  For $c$ having a uniform distribution over the sphere, by Lemma~\ref{lem:helper1}, the same equality holds approximately, 

\begin{equation}
\Exp_{c}[\expinner{v_w+v_{w'},c}] = (1\pm \epsilon_3)\exp(\|v_w+v_{w'}\|^2/(2d)) \label{eqn:eqn18}
\end{equation}
where $\epsilon_3  = \widetilde{O}(1/d)$. 

%
%
%
%
%
%

Plugging in equation~\eqref{eqn:eqn18} into equation~\eqref{eqn:eqn19}, we have that 
\begin{align*}
	\log p(w,w') 
	& = \log\left(  (1\pm \epsilon_3)\exp(\|v_w+v_{w'}\|^2/(2d)) \pm \delta_0 \right)- 2\log Z + 2\log(1\pm \epsilon_z) + \log(1\pm \epsilon_2)\\
	&  = \|v_w + v_{w'}\|^2/(2d) + O(\epsilon_3) + O(\delta_0') - 2\log Z \pm 2\epsilon_z\pm \epsilon_2\\
\end{align*}

where $\delta_0' = \delta_0\cdot \left(\Exp_{c\sim \mathcal{C}}[\expinner{v_w+v_{w'},c}]\right)^{-1}= \exp(-\Omega(\log^{1.8} n))$. 
Note that $\epsilon_3 = \widetilde{O}(1/d)$, $\epsilon_z = \widetilde{O}(1/\sqrt{n})$, and $\epsilon_2$ by assumption, therefore we obtain that

\begin{equation*}
\log p(w,w')   = \frac{1}{2d}\|v_w + v_{w'}\|^2 - 2\log Z \pm O(\epsilon_z) + O(\epsilon_2)+ \widetilde{O}(1/d).
\end{equation*}


\ 
\end{proof}

The following lemmas are helper lemmas that were used in the proof above. We use $\mathcal{C}^d$ to denote the uniform distribution over the unit sphere in $\R^d$.

\begin{lem} [Tail bound for spherical distribution] If $c \sim \mathcal{C}^d$, $v \in \mathbb{R}^d$ is a vector with $\|v\| = \Omega(\sqrt{d})$ and $t = \omega(1)$, the random variable $z = \langle v, c \rangle$ satisfies $\Pr[z \geq t] = e^{-O(t^2)}$.  
\label{l:sphericalangle} 
\end{lem} 
\begin{proof} 
If $c=(c_1, c_2, \dots, c_d) \sim \mathcal{C}^d$, $c$ is in distribution equal to $\left( \frac{\tilde{c}_1}{\|\tilde{c}\|}, \frac{\tilde{c}_2}{\|\tilde{c}\|}, \dots, \frac{\tilde{c}_d}{\|\tilde{c}\|} \right)$ where the $\tilde{c}_i$ are i.i.d. samples from a univariate Gaussian with mean 0 and variance $\frac{1}{d}$. 
By spherical symmetry, we may assume that $v = (\|v\|, 0, \dots, 0)$. 
Let's introduce the random variable $r = \sum_{i=2}^d \tilde{c}^2_i$. 
Since 
\begin{equation*}  
\Pr\left[\langle v,c \rangle \geq t\right] = \Pr\left[\|v| \frac{\tilde{c}_1}{\|\tilde{c}\|} \geq t\right] \le \Pr\left[\frac{\|v\|\tilde{c}_1}{\|\tilde{c}\|} \geq t \mid r \ge \frac{1}{2}\right] \Pr\left[r \ge \frac{1}{2}\right]  + \Pr\left[\frac{\|v\|\tilde{c}_1}{\|\tilde{c}\|} \geq t \mid r \ge \frac{1}{2}\right] \Pr\left[r \ge \frac{1}{2}\right] 
\end{equation*}   

it's sufficient to lower bound $\Pr\left[r \leq 100\right]$ and $\Pr\left[\|v\| \frac{\tilde{c}_1}{\|c\|} \geq t \mid r \leq 100\right]$. The former probability is easily seen to be lower bounded by a constant by a Chernoff bound. Consider the latter one next. It holds that
\begin{equation*}
\Pr\left[\|v\| \frac{\tilde{c}_1}{\|\tilde{c}\|} \geq t \mid r \leq 100\right] = \Pr\left[\tilde{c}_1 \geq \sqrt{\frac{t^2 \cdot r}{\|v\|^2-t^2}} \mid r \leq 100\right] \geq \Pr\left[\tilde{c}_1 \geq \sqrt{\frac{100 t^2}{\|v\|^2-t^2}}\right]
\end{equation*}
Denoting $\tilde{t} = \sqrt{\frac{100 t^2}{\|v\|^2-t^2}}$, by a well-known Gaussian tail bound it follows that
$$\Pr\left[\tilde{c}_1 \geq \tilde{t} \right] = e^{-O(d \tilde{t}^2)}\left(\frac{1}{\sqrt{d} \tilde{t}} - \left(\frac{1}{\sqrt{d} \tilde{t}}\right)^3\right) = e^{-O(t^2)} $$ 

where the last equality holds since $\|v\| = \Omega(\sqrt{d})$ and $t = \omega(1)$. 
\end{proof} 

\begin{lem} If $c \sim \mathcal{C}^d$, $v \in \mathbb{R}^d$ is a vector with $\|v\| = \Theta(\sqrt{d})$ and $t = \omega(1)$, the random variable $z = \langle v,c \rangle$ satisfies 
$$\Exp\left[ \exp(z)\indicator{[t,+\infty]}(z)\right]  = \exp(-\Omega(t^2)) + \exp(-\Omega(d)) $$
\label{l:abelspherical} 
\end{lem} 
\begin{proof}

Similarly as in Lemma~\ref{l:sphericalangle}, if $c=(c_1, c_2, \dots, c_d) \sim \mathcal{C}^d$, $c$ is in distribution equal to $\left( \frac{\tilde{c}_1}{\|\tilde{c}\|}, \frac{\tilde{c}_2}{\|\tilde{c}\|}, \dots, \frac{\tilde{c}_d}{\|\tilde{c}\|} \right)$ where the $\tilde{c}_i$ are i.i.d. samples from a univariate Gaussian with mean 0 and variance $\frac{1}{d}$. Again, by spherical symmetry, we may assume $v = (\|v\|, 0, \dots, 0)$. Let's introduce the random variable $r = \sum_{i=2}^d \tilde{c}^2_i$. Then,
for an arbitrary $u > 1$, some algebraic manipulation shows 
$$ \Pr\left[ \exp\left(\langle v,c \rangle\right)\indicator{[t,+\infty]}(\langle v, c\rangle)\geq u\right] = \Pr\left[ \exp\left(\langle v,c \rangle\right) \geq u \wedge \langle v, c\rangle \geq t \right] = $$ 
\begin{equation} 
\Pr\left[\exp\left(\|v\| \frac{\tilde{c}_1}{\|\tilde{c}\|}\right) \geq u \wedge \|v\| \frac{\tilde{c}_1}{\|\tilde{c}\|} \geq u \right] =  \Pr\left[\tilde{c}_1 = \max\left(\sqrt{\frac{\tilde{u}^2 r}{\|v\|^2 - \tilde{u}^2}}, \sqrt{\frac{t^2 r}{\|v\|^2 - t^2}}\right) \right] 
\label{eq:abelexpr1}
\end{equation}  
where we denote $\tilde{u} = \log u$. Since $\tilde{c}_1$ is a mean 0 univariate Gaussian with variance $\frac{1}{d}$, and $\|v\| = \Omega(\sqrt{d})$ we have $\forall x \in \mathbb{R}$ 
$$\Pr\left[\tilde{c}_1 \geq \sqrt{\frac{x^2 r}{\|v\|^2 - u^2}} \right] = O\left(e^{-\Omega(x^2 r)}\right)$$ 
Next, we show that $r$ is lower bounded by a constant with probability $1-\exp(-\Omega(d))$. Indeed, $r$ is in distribution equal to $\frac{1}{d} \chi^2_{d-1}$, where $\chi^2_k$ is a Chi-squared distribution with $k$ degrees of freedom. Standard concentration bounds \cite{laurent2000adaptive} imply that $\forall \xi \geq 0, \Pr[r-1 \leq -2\sqrt{\frac{\xi}{d}}] \leq \exp(-\xi)$. Taking $\xi = \alpha d$ for $\alpha$ a constant implies that with probability $1-\exp(-\Omega(d))$,  $r \geq M$ for some constant $M$. 
We can now rewrite
$$\Pr\left[\tilde{c}_1 \geq \sqrt{\frac{x^2 r}{\|v\|^2 - x^2}} \right] =  $$
$$ \Pr\left[\tilde{c}_1 \geq \sqrt{\frac{x^2 r}{\|v\|^2 - x^2}}  \mid r \geq M\right] \Pr[r \geq M] + \Pr\left[\tilde{c}_1 \geq \sqrt{\frac{x^2 r}{\|v\|^2 - x^2}} \mid r \leq M\right] \Pr[r \leq M]$$  
The first term is clearly bounded by $ e^{-\Omega(x^2)}$ and the second by $\exp(-\Omega(d))$. 
Therefore, 
\begin{equation} 
\Pr\left[\tilde{c}_1 \geq \sqrt{\frac{x^2 r}{\|v\|^2 - x^2}} \right] = O\left(\max\left(\exp\left(-\Omega\left(x^2\right)\right), \exp\left(-\Omega\left(d\right)\right)\right)\right) 
\label{eq:abelexpr2}
\end{equation} 
Putting \ref{eq:abelexpr1} and \ref{eq:abelexpr2} together, we get that 
\begin{equation} 
\Pr\left[ \exp\left(\langle v,c \rangle\right)\indicator{[t,+\infty]}(\langle v, c\rangle)\geq u\right] = O\left(\max\left(\exp\left(-\Omega\left(\min\left(d, \left(\max\left(\tilde{u},t\right)\right)^2\right)\right)\right)\right)\right) 
\label{eq:abelexpr3}
\end{equation}
(where again, we denote $\tilde{u} = \log u$) 

For any random variable $X$ which has non-negative support, it's easy to check that 
$$\Exp[X] = \int_{0}^{\infty} \Pr[X \geq x] dx $$ 
Hence, 
$$\Exp\left[ \exp(z)\indicator{[t,+\infty]}(z)\right] = \int_{0}^{\infty} \Pr\left[\exp(z)\indicator{[t,+\infty]}(z) \geq u \right] du = \int_{0}^{\exp(\|v\|)} \Pr\left[\exp(z)\indicator{[t,+\infty]}(z) \geq u \right] du $$
To bound this integral, we split into the following two cases: 
\begin{itemize} 
\item Case $t^2 \geq d$: $\max\left(\tilde{u},t\right) \geq t$, so $\min\left(d, \left(\max\left(\tilde{u},t\right)\right)^2\right) = d$. Hence, \ref{eq:abelexpr3} implies 
$$\Exp\left[ \exp(z)\indicator{[t,+\infty]}(z)\right] = \exp(\|v\|)\exp(-\Omega(d)) = \exp(-\Omega(d))  $$
where the last inequality follows since $\|v\| = O(\sqrt{d})$. 
\item Case $t^2 < d$: In the second case, we will split the integral into two portions: 
$u \in [0, \exp(t)]$ and $u \in [\exp(t), \exp(\|v\|)]$. 

When $u \in [0, \exp(t)]$, $\max\left(\tilde{u},t\right)= t$, so $\min(d,\left(\max\left(\tilde{u},t\right)\right)^2) = t^2$. Hence, 
$$ \int_{0}^{\exp(t)} \Pr\left[\exp(z)\indicator{[t,+\infty]}(z) \geq u \right] du \leq \exp(t) \exp(-\Omega(t^2)) = -\exp(\Omega(t^2))$$ 

When $u \in [\exp(t), \exp(\|v\|)]$, $\max\left(\tilde{u},t\right)= \tilde{u}$. 
But $\tilde{u} \leq \log(\exp(\|v\|)) = O(\sqrt{d})$, so $\min(d,\left(\max\left(\tilde{u},t\right)\right)^2) = \tilde{u}$. Hence, 
$$ \int_{\exp(t)}^{\exp(\|v\|)} \Pr\left[\exp(z)\indicator{[t,+\infty]}(z) \geq u \right] du \leq \int_{\exp(t)}^{\exp(\|v\|)} \exp(-(\log(u))^2) du$$
Making the change of variable $\tilde{u} = \log(u)$, the we can rewrite the last integral as
$$ \int_{t}^{\|v\|} \exp(-\tilde{u}^2) \exp(\tilde{u}) d\tilde{u} = O(\exp(-t^2)) $$ 
where the last inequality is the usual Gaussian tail bound. 
\end{itemize} 
In either case, we get that 
$$\int_{0}^{\exp(\|v\|)} \Pr\left[\exp(z)\indicator{[t,+\infty]}(z) \geq u \right] du =  \exp(-\Omega(t^2)) + \exp(-\Omega(d))) $$ 
which is what we want.
\end{proof} 

As a corollary to the above lemma, we get the following: 

\begin{cor} If $c \sim \mathcal{C}^d$, $v \in \mathbb{R}^d$ is a vector with $\|v\| = \Theta(\sqrt{d})$ and 
$t = \Omega(\log^{.9}n)$ then  
$$\Exp\left[ \exp(z)\indicator{[t,+\infty]}(z)\right]  = \exp(-\Omega(\log^{1.8}n)) $$
\label{c:abelspherical} 
\end{cor} 
\begin{proof}
We claim the proof is trivial if $d = o(\log^{4}n)$. Indeed, in this case, $\exp(\langle v,c\rangle) \le \exp(\|v\|) = \exp(O(\sqrt{d}))$. Hence, 
$$\Exp\left[ \exp(z)\indicator{[t,+\infty]}(z)\right]  = \exp(O(\sqrt{d})) \Exp[\indicator{[t,+\infty]}(z)] = \exp(O(\sqrt{d})) \Pr[z \geq t] $$
Since by Lemma~\ref{l:sphericalangle}, $\Pr[z \geq t] \le \exp(-\Omega(\log^2 n)$, we get
$$\Exp\left[ \exp(z)\indicator{[t,+\infty]}(z)\right] = \exp(O(\sqrt{d}) - \Omega(\log^{2}n))  = \exp(-\Omega(\log^{1.8}n))$$ 
as we wanted. 

So, we may, without loss of generality assume that $d = \Omega(\log^{4}n)$. In this case, Lemma~\ref{l:abelspherical} implies  
$$\Exp\left[ \exp(z)\indicator{[t,+\infty]}(z)\right]  = \exp(-\log^{1.8}n) + \exp(-\Omega(d)))  = \exp(-\log^{1.8}n)$$ 
where the last inequality holds because $d = \Omega(\log^{4}n)$ and $t^2 = \Omega(\log^{.9}n)$, so we get the claim we wanted.   

\end{proof} 
%
%
%
%

%

\begin{lem}[Continuous Abel's Inequality]\label{lem:abel}
	Let $0\le r(x)\le 1$ be a function such that such that $\Exp[r(x)] = \rho$. Moreover, suppose increasing function $u(x)$ satisfies that $\Exp[|u(x)|] < \infty$. Let $t$ be the real number  such that $\Exp[\indicator{[t,+\infty]}] = \rho$. 
	Then we have 
		\begin{equation}
		\Exp[u(x)r(x)]\le \Exp[u(x)\indicator{[t,+\infty]}]
		\end{equation}
\end{lem}

\begin{proof}
Let $G(z) = \int_{z}^{\infty} f(x)r(x) dx$, and $H(z) = \int_{z}^{\infty} f(x)\indicator{[t,+\infty]}(x)dx$. Then we have that $G(z)\le H(z)$ for all $z$. Indeed, for $z\ge t$, this is trivial since $r(z)\le 1$. For $z\le t$, we have  $H(z) = \Exp[\indicator{[t,+\infty]}] = \rho = \Exp[r(x)]\ge \int_{z}^{\infty} f(x)r(x) dx.$ Then by integration by parts we have, 
	\begin{align*}
				\int_{-\infty}^{\infty} u(x)f(x)r(x)dx& 
				=- \int_{-\infty}^{\infty} u(x) dG \\
&				= -u(x)G(x)\mid_{-\infty}^{\infty} + \int_{-\infty}^{+\infty} G(x)u'(x)dx \\
&				\le   \int_{-\infty}^{+\infty} H(x)u'(x)dx \\
& = \int_{-\infty}^{\infty} u(x)f(x)\indicator{[t,+\infty]}(x)dx, 
		\end{align*}
		where at the third line we use the fact that $u(x)G(x)\rightarrow 0$ as $x\rightarrow \infty$ and that $u'(x)\ge 0$, and at the last line we integrate by parts again. 
\end{proof}

\begin{lem}\label{lem:helper1}
	Let $v\in \R^d$ be a fixed vector with norm $\|v\| \le \kappa \sqrt{d}$ for absolute constant $\kappa$. Then for random variable $c$ with uniform distribution over the sphere, we have that 
	\begin{equation}
	\log \Exp[\exp(\inner{v}{d})] = \|v\|^2/{2d} \pm \epsilon
	\end{equation}
	where $\epsilon = \widetilde{O}(\frac{1}{d})$. 
\end{lem}

\begin{proof}
	Let $g\in \mathcal{N}(0,I)$, then $g/\|g\|$ has the same distribution as $c$. Let $r = \|v\|$. Since $c$ is spherically symmetric, we could, we can assume without loss of generality that $v = (r,0,\dots,0)$. Let $x= g_1$ and $y = \sqrt{g_2^2+\dots+g_d^2}$. Therefore $x\in \mathcal{N}(0,1)$ and $y^2$ has $\chi^2$ distribution with mean $d-1$ and variance $O(d)$.  
	
	Let $\mathcal{F}$ be the event that $x\le 20 \log d$ and $1.5\sqrt{d}\ge y\ge 0.5\sqrt{d}$. Note that the $\Pr[\mathcal F] \ge 1- \exp(-\Omega(\log^{1.8}(d))$. By Proposition~\ref{prop:helper4}, we have that 
	$ \Exp[\exp(\inner{v}{c})]  = \Exp[\exp(\inner{v}{c})\mid \mathcal{F}]\cdot (1\pm \Omega(-\log^{1.8}d))$. 
	
	Conditioned on event $\mathcal F$, we have 
	\begin{align}
		\Exp[\exp(\inner{v}{c})\mid \mathcal{F}] & = \Exp\left[\exp(\frac{rx}{\sqrt{x^2+y^2}})\mid \mathcal{F}\right]\nonumber\\
		& = \Exp\left[\exp(\frac{rx}{y} -\frac{rx^3}{y\sqrt{x^2+y^2}(y+\sqrt{x^2+y^2})} )\mid \mathcal{F}\right]\nonumber\\
		& = \Exp\left[\exp(\frac{rx}{y})\cdot \exp(\frac{rx^3}{y\sqrt{x^2+y^2}(y+\sqrt{x^2+y^2})} )\mid \mathcal{F}\right]\nonumber\\
		& = \Exp\left[\exp(\frac{rx}{y})\mid \mathcal F\right] \cdot (1\pm O(\frac{\log^{3} d}{d}) )\label{eqn:eqn17}
	\end{align}
	where we used the fact that $r\le \kappa \sqrt{d}$. Let $\mathcal{E}$ be the event that $1.5\sqrt{d}\ge y\ge 0.5\sqrt{d}$. By using Proposition~\ref{prop:helper3}, we have that 
	\begin{equation}
	\Exp[\exp(rx/y)\mid \mathcal F] =
	\Exp[\exp(rx/y)\mid \mathcal{E}] \pm \exp(-\Omega(\log^2(d))
	\end{equation} 
	Then let $z = y^2/(d-1)$ and $w = z-1$. Therefore $z$ has $\chi^2$ distribution with mean 1 and variance $1/(d-1)$, and $w$ has mean 0 and variance $1/(d-1)$. 
	\newcommand{\calF}{\mathcal F}
	\begin{align*}
		\Exp\left[\exp(\frac{rx}{y})\mid \mathcal{E}\right] &= \Exp[\Exp[\exp(rx/y)\mid y]\mid \mathcal{E}]= \Exp[\exp(r^2/y^2)\mid \mathcal{E}] \\
		&  = \Exp[\exp(r^2/(d-1) \cdot 1/z^2 )\mid \mathcal{E}]\\ 
		& = \Exp[\exp(r^2/(d-1) \cdot (1+ \frac{2w+w^2}{(1+w)^2}))\mid \mathcal E]\\ 
		& = \exp(r^2/(d-1)) \Exp[\exp( 1+ \frac{2w+w^2}{(1+w)^2}))\mid \mathcal{E}]\\ 
		& = \exp(r^2/(d-1)) \Exp[ 1+ 2w \pm O(w^2)\mid \mathcal{E}]\\ 
		& = \exp(r^2/(d-1)^2) (1\pm 1/d)
	\end{align*}
	
	where the second-to-last line uses the fact that conditioned on $1/2\ge \mathcal E$, $w\ge -1/2$ and therefore the Taylor expansion approximates the exponential accurately,  and 
	the last line uses the fact that $|\Exp[w\mid \mathcal{E}]| = O(1/d)$ and $\Exp[w^2 \mid \mathcal E] \le O(1/d)$. Combining the series of approximations above completes the proof. 
	
\end{proof}

We finally provide the proofs for a few helper propositions on conditional probabilities for high probability events used in the lemma above.  

\begin{proposition}\label{prop:helper3}
	Suppose $x\sim \mathcal{N}(0,\sigma^2)$ with $\sigma = O(1)$. Then for any event $\mathcal E$ with $\Pr[\mathcal E] = 1-O(-\Omega(\log^2 d))$, we have that $\Exp[\exp(x)] = \Exp[\exp(x)\mid \mathcal{E}] \pm \exp(-\Omega(\log^2(d))$. 
\end{proposition}
\begin{proof}
Let's denote by $\bar{\mathcal{E}}$ the complement of the event $\mathcal{E}$. We will consider the upper and lower bound separately. 
Since 
$$\Exp[\exp(x)] = \Exp[\exp(x) \mid \mathcal{E}] \Pr[\mathcal{E}] + \Exp[\exp(x) \mid \bar{\mathcal{E}}] \Pr[\bar{\mathcal{E}}] $$
we have that 
\begin{equation} \Exp[\exp(x)] \le \Exp[\exp(x) \mid \mathcal{E}]  + \Exp[\exp(x) \mid \bar{\mathcal{E}}] \Pr[\bar{\mathcal{E}}] \label{eq:lbh3} \end{equation} 
and 
\begin{equation} \Exp[\exp(x)] \ge \Exp[\exp(x) \mid \mathcal{E}] (1-\exp(-\Omega(\log^2 d))) \ge \Exp[\exp(x) \mid \mathcal{E}] - \Exp[\exp(x) \mid \mathcal{E}] \exp(-\Omega(\log^2 d))) \label{eq:uph3} \end{equation}

Consider the upper bound \eqref{eq:lbh3} first. To show the statement of the lemma, it suffices to bound $\Exp[\exp(x) \mid \bar{\mathcal{E}}] \Pr[\bar{\mathcal{E}}]$.

Working towards that, notice that 
$$\Exp[\exp(x)\mid \bar{\mathcal{E}}]\Pr[\bar{\mathcal{E}}]= \Exp[\exp(x)\mathbf{1}_{\bar{\mathcal {E}}}] = \Exp[\exp(x)\Exp[\mathbf{1}_{\bar{\mathcal {E}}}\vert x]] =  \Exp[\exp(x)r(x)]$$
if we denote $r(x) = \Exp[\mathbf{1}_{\bar{\mathcal{E}}}\vert x]$. We wish to upper bound $\Exp[\exp(x) r(x)]$. By Lemma \ref{lem:abel}, we have 
$$ \Exp[\exp(x) r(x)] \le \Exp[\exp(x) \mathbf{1}_{[t, \infty]}]$$   
where $t$ is such that $\Exp[\mathbf{1}_{[t, \infty]}] = \Exp[r(x)]$. However, since $\Exp[r(x)] = \Pr[\bar{\mathcal{E}}] = \exp(-\Omega(\log^2 d))$, it must be the case that $t = \Omega(\log d)$ by the standard Gaussian tail bound, and the assumption that $\sigma = O(1)$. 
In turn, this means 
$$\Exp[\exp(x) \mathbf{1}_{[t, \infty]}] \le \frac{1}{\sigma \sqrt{2 \pi}}\int_{t}^{\infty} e^{x}e^{-\frac{x^2}{\sigma^2}} dx = 
\frac{1}{\sigma \sqrt{2 \pi}}\int_{t}^{\infty} e^{-(\frac{x}{\sigma} - \frac{\sigma}{2})^2 + \frac{\sigma^2}{4}} dx = 
e^{\frac{\sigma^2}{4}}  \frac{1}{\sqrt{2 \pi}}\int_{t/\sigma}^{+\infty} e^{-(x' - \frac{\sigma}{2})^2} dx' $$ 
where the last equality follows from the change of variables $x = \sigma x'$. However, 
$$ \frac{1}{\sqrt{2 \pi}}\int_{t/\sigma}^{+\infty} e^{-(x' - \frac{\sigma}{2})^2} dx' $$ 
 is nothing more than $\Pr[x' > \frac{t}{\sigma}]$, where $x'$ is distributed like a univariate gaussian with mean $\frac{\sigma}{2}$ and variance 1. Bearing in mind that $\sigma = O(1)$
$$ e^{\frac{\sigma^2}{4}}  \frac{1}{\sqrt{2 \pi}}\int_{t/\sigma}^{+\infty} e^{-(x' - \frac{\sigma}{2})^2} dx' = \exp(-\Omega(t^2)) = \exp(-\Omega(\log^2d))$$ 
by the usual Gaussian tail bounds, which proves the lower bound we need. 

We proceed to consider the lower bound \ref{eq:uph3}. To show the statement of the lemma, we will bound $\Exp[\exp(x) \mid \mathcal{E}]$. Notice trivially that since $\exp(x) \ge 0$, 
$$\Exp[\exp(x) \mid \mathcal{E}] \le \frac{\Exp[\exp(x)]}{\Pr[\mathcal{E}]}  $$ 
Since $\Pr[\mathcal{E}] \ge 1 - \exp(\Omega(\log^2 d))$, $\frac{1}{\Pr[\mathcal{E}]} \le 1 + \exp(O(\log^2))$. So, it suffices to bound $\Exp[\exp(x)]$. However, 
$$\Exp[\exp(x)] = \frac{1}{\sigma \sqrt{2 \pi}}\int_{t = -\infty}^{+\infty} e^{x} e^{-\frac{x^2}{\sigma^2}} dx = \frac{1}{\sigma \sqrt{2 \pi}}\int_{t = -\infty}^{+\infty} e^{-(\frac{x}{\sigma} - \frac{\sigma}{2})^2 + \frac{\sigma^2}{4}} dx = \frac{1}{\sqrt{2 \pi}} \int_{t = -\infty}^{+\infty} e^{-(x' - \frac{\sigma}{2})^2 + \frac{\sigma^2}{4}} dx'$$
where the last equality follows from the same change of variables $x = \sigma x'$ as before. Since 
$ \int_{t = -\infty}^{+\infty} e^{-(x' - \frac{\sigma}{2})^2} dx' = \sqrt{2 \pi} $, we get  
$$\frac{1}{\sqrt{2 \pi}} \int_{t = -\infty}^{+\infty} e^{-(x' - \frac{\sigma}{2})^2 + \frac{\sigma^2}{4}} dx' = e^{\frac{\sigma^2}{4}} = O(1) $$ 
Putting together with the estimate of $\frac{1}{\Pr[\mathcal{E}]}$, we get that $\Exp[\exp(x) \mid \mathcal{E}] = O(1)$. Plugging this back in \ref{eq:uph3}, we get the desired upper bound. 
\end{proof} 
\begin{proposition}\label{prop:helper4}
	Suppose  $c \sim \mathcal{C}$ and $v$ is an arbitrary vector with $\|v\| = O(\sqrt{d})$. Then for any event $\mathcal E$ with $\Pr[\mathcal E]\ge 1-\exp(-\Omega(\log^2 d))$, we have that $\Exp[\exp(\inner{v}{c})] = \Exp[\exp(\inner{v}{c})\mid \mathcal{E}] \pm \exp(-\log^{1.8} d)$. 
\end{proposition}

\begin{proof}[Proof of Proposition~\ref{prop:helper4}]
	Let $z = \inner{v}{c}$. We proceed similarly as in the proof of Proposition~\ref{prop:helper3}. We have
	$$ \Exp[\exp(z)] = \Exp[\exp(z)\mid \mathcal{E}]\Pr[\mathcal E] + \Exp[\exp(z)\mid \bar{\mathcal{E}}]\Pr[\bar{\mathcal E}]  $$ 
	and 
	\begin{equation} \Exp[\exp(z)] \le \Exp[\exp(z)\mid \mathcal{E}] + \Exp[\exp(z)\mid \bar{\mathcal{E}}]\Pr[\bar{\mathcal E}] \label{eq:uph4} \end{equation}
	and 
	\begin{equation} \Exp[\exp(z)] \ge \Exp[\exp(z)\mid \mathcal{E}]\Pr[\mathcal E] = \Exp[\exp(z)\mid \mathcal{E}] - \Exp[\exp(z)\mid \mathcal{E}] \exp(-\Omega(\log^2d))\label{eq:lwrh4} \end{equation}  
	We again proceed by separating the upper and lower bound. 
%
%

We first consider the upper bound \ref{eq:uph4}. 

Notice that that 
$$\Exp[\exp(z)\mid \bar{\mathcal{E}}]\Pr[\bar{\mathcal E}]= \Exp[\exp(z)\mathbf{1}_{\bar{\mathcal E}}] $$

We can split the last expression as 
\begin{equation*}
\Exp \left[\expinner{v_w,c} \mathbf{1}_{\langle v_w,c \rangle > 0} \mathbf{1}_{\overline{\mathcal{E}}} \right] + \Exp \left[ \expinner{v_w,c} \mathbf{1}_{\langle v_w,c \rangle < 0} \mathbf{1}_{\overline{\mathcal{E}}} \right]\,.
\end{equation*}
The second term is upper bounded by 
$$\Exp [\mathbf{1}_{\overline{\mathcal{E}}}]\le \exp(-\Omega(\log^2 n))$$ 
We proceed to the first term of (\ref{eqn:eqn21}) and observe the following property of it:
$$ \Exp \left[\expinner{v_w,c} \mathbf{1}_{\langle v_w,c \rangle > 0} \mathbf{1}_{\overline{\mathcal{E}}}\right] \leq \Exp \left[\expinner{\alpha v_w,c} \mathbf{1}_{\langle v_w,c \rangle > 0}\ \mathbf{1}_{\overline{\mathcal{E}}} \right] \leq \Exp \left[\expinner{\alpha v_w,c} \mathbf{1}_{\overline{\mathcal{E}}} \right]$$
where $\alpha > 1$. Therefore, it's sufficient to bound
$$\Exp \left[ \exp(z) \mathbf{1}_{\overline{\mathcal{E}}} \right] $$
when $\|v_w\| = \Theta(\sqrt{d})$.   
Let's denote $r(z) = \Exp[\mathbf{1}_{\bar{\mathcal E}}\vert z]$. 

%
		Using Lemma~\ref{lem:abel}, we have that
		
		\begin{equation}
		\Exp_c[\exp(z)r(z)]\le \Exp\left[ \exp(z)\indicator{[t,+\infty]}(z)\right]
		\end{equation}
		where $t$ satisfies that 
		$\Exp_c[\mathbf{1}_{[t,+\infty]}] = \Pr[z\ge t] = \Exp_c[r(z)] \leq \exp(-\Omega(\log^2 d))$.
		Then, we claim $\Pr[z\ge t] \le \exp(-\Omega(\log^2 d))$ implies that $t \ge \Omega(\log^{.9} d)$.
		
		Indeed, this follows by directly applying Lemma \ref{l:sphericalangle}. 
		%
		%
		Afterward, applying Lemma \ref{l:abelspherical}, we have: 
		\begin{equation}
		\Exp[\exp(z)r(z)]\le \Exp\left[ \exp(z)\indicator{[t,+\infty]}(z)\right] = \exp(-\Omega(\log^{1.8}d))
		\end{equation}
which proves the upper bound we want. 
		
We now proceed to the lower bound \ref{eq:lwrh4}, which is again similar to the lower bound in the proof of Proposition~\ref{prop:helper3}: we just need to bound $\Exp[\exp(z)\mid \mathcal{E}]$. Same as in Proposition~\ref{prop:helper3}, since $\exp(x) \ge 0$, 
$$\Exp[\exp(z) \mid \mathcal{E}] \le \frac{\Exp[\exp(z)]}{\Pr[\mathcal{E}]}  $$ 
Consider the event $\mathcal{E}':z \leq t$, for $t = \Theta(\log^{.9}d)$, which by Lemma \ref{l:sphericalangle} satisfies $\Pr[\mathcal{E}'] \ge 1 - \exp(-\Omega(\log^2d))$. By the upper bound we just showed, 
$$\Exp[\exp(z)] \le \Exp[\exp(z) \mid \mathcal{E}'] + \exp(-\Omega(\log^2n)) = O(\exp(\log^{.9}d)) $$ 
where the last equality follows since conditioned on $\mathcal{E}'$, $z = O(\log^{.9}d)$. Finally, this implies  
$$\Exp[\exp(z) \mid \mathcal{E}] \le \frac{1}{\Pr[\mathcal{E}]}  O(\exp(\log^{.9}d)) = O(\exp(\log^{.9}d)) $$ 
where the last equality follows since $\Pr[\mathcal{E}] \ge 1 - \exp(-\Omega(\log^2n))$. Putting this together with \ref{eq:lwrh4}, we get	
$$ \Exp[\exp(z)] \ge \Exp[\exp(z)\mid \mathcal{E}]\Pr[\mathcal E] = \Exp[\exp(z)\mid \mathcal{E}] - \Exp[\exp(z)\mid \mathcal{E}] \exp(-\Omega(\log^2d)) \ge $$
$$ \Exp[\exp(z)\mid \mathcal{E}] - O(\exp(\log^{.9}d))\exp(-\Omega(\log^2d)) \ge \Exp[\exp(z)\mid \mathcal{E}]  - \exp(-\Omega(\log^2d))$$
which is what we needed.   	
	\end{proof}

\subsection{Analyzing partition function $Z_c$}\label{sec:zc}

In this section, we prove Lemma~\ref{thm:Z_c}. We basically first prove that for the means of $Z_c$ are all $(1+o(1))$-close to each other, and then prove that $Z_c$ is concentrated around its mean. It turns out the concentration part is non trivial because the random variable of concern, $\expinner{v_w,c}$ is not well-behaved in terms of the tail. Note that $\expinner{v_w,c}$ is NOT sub-gaussian for any variance proxy. 
This essentially disallows us to use an existing concentration inequality directly. We get around this issue by considering the truncated version of $\expinner{v_w,c}$, which is bounded, and have similar tail properties as the original one, in the regime that we are concerning. 

We bound the mean and variance of $Z_c$ first in the Lemma below. 

%

\begin{lem}\label{lem:exp_var_Z_c}
	For any fixed unit vector $c\in \R^d$,  we have that $\Exp[Z_c]\ge n$ and $\mathbb{V}[Z_c] \le O(n)$. 
\end{lem}

\begin{proof}[Proof of Lemma~\ref{lem:exp_var_Z_c}]
		
		Recall that by definition $$Z_c = \sum_{w}\exp(\langle v_w, c\rangle).$$ We fix context $c$ and view $v_w$'s as random variables throughout this proof. 
		Recall that $v_w$ is composed of $v_w = s_w \cdot \hat{v}_w$, where $s_w$ is the scaling and $\hat{v}_w$ is from spherical Gaussian with identity covariance $I_{d\times d}$. Let $s$ be a random variable that has the same distribution as $s_w$.  
	
	We lowerbound the mean of $Z_c$ as follows: 
	
	$$\Exp[Z_c] =n\Exp\left[\exp(\langle v_w,c\rangle)\right]\ge n\Exp\left[1+\langle v_w,c\rangle\right] = n$$
	where the last equality holds because of the symmetry of the spherical Gaussian distibution.  
	On the other hand, to upperbound the mean of $Z_c$, we  condition on the scaling $s_w$, 
	\begin{align*}
		\E[Z_c] &= n\Exp[\exp(\langle v_w, c\rangle)] \\
		&=n \Exp\left[\Exp\left[\exp(\langle v_w, c\rangle)\mid s_w\right]\right]
	\end{align*}
	
	Note that conditioned on $s_w$, we have that $\langle v_w, c\rangle$ is a Gaussian random variable with variance 
	$\sigma^2 = s_w^2$. Therefore,
	\begin{align*}
		\Exp\left[\exp(\langle v_w, c\rangle)\mid s_w\right] &= \int_x \frac{1}{\sigma\sqrt{2\pi}} \exp(-\frac{x^2}{2\sigma^2}) \exp(x)dx \\
		&=\int_x \frac{1}{\sigma\sqrt{2\pi}} \exp(-\frac{(x-\sigma^2)^2}{2\sigma^2} + \sigma^2/2)dx \\
		&= \exp(\sigma^2/2) 
	\end{align*} 
	
	
	It follows that $$\Exp[Z_c] = n\Exp[\exp(\sigma^2/2) ] = n\Exp[\exp(s_w^2/2)]= n\Exp[\exp(s^2/2)].$$
	%
	%
	
	We calculate the variance of $Z_c$ as follows: 
	
	\begin{align*}
		\mathbb{V}[Z_c] & = \sum_w\mathbb{V}\left[\exp(\langle v_w, c\rangle)\right] \le n\Exp[\exp(2\langle v_w, c\rangle)] \\
		&=n \Exp\left[\Exp\left[\exp(2\langle v_w, c\rangle)\mid s_w\right]\right]
	\end{align*}
	By a very similar calculation as above, using the fact that $2\langle v_w,c\rangle$ is a Gaussian random variable with variance 
	$4 \sigma^2 = 4 s_w^2$, 
	\begin{align*}
		\Exp\left[\exp(2\langle v_w, c\rangle)\mid s_w\right] 
		&= \exp(2\sigma^2) 
	\end{align*}

	Therefore, we have that

	
	\begin{align*}
	\mathbb{V}[Z_c]&\le n \Exp\left[\Exp\left[\exp(2\langle v_w, c\rangle)\mid s_w \right]\right]\\
	& = n\Exp\left[\exp(2\sigma^2)\right] =  n\Exp\left[\exp(2s^2)\right]\le \Lambda n
  \end{align*}
	
	for $\Lambda = \exp(8 \kappa^2)$ a constant, and at the last step we used the facts that $s\le \kappa$ a.s.
\end{proof}

Now we are ready to prove Lemma~\ref{thm:Z_c}. 
\begin{proof}[Proof of Lemma~\ref{thm:Z_c}]
	
	We fix the choice of $c$, and the proving the concentration using the randomness of $v_w$'s first. 
		Note that that $\exp(\langle v_w,c\rangle) $ is neither sub-Gaussian nor sub-exponential (actually the Orlicz norm of random variable $\exp(\langle v_w,c\rangle) $ is never bounded).
		This prevents us from applying the usual concentration inequalities. The proof deals with this issue in a slightly more specialized manner. 
		
		Let's define $\mathcal{F}_w$ be the event that $|\langle v_w,c\rangle|\le \frac{1}{2}\log  n$. 
		We claim that $\Pr[\mathcal{F}_w] \ge 1-\exp(-\Omega(\log^2n))$. Indeed note that $\inner{v_w}{c} \mid s_w$ has a Gaussian distribution with standard deviation 
		$s_w \|c\| = s_w \leq 2\kappa$ a.s. 
		Therefore by the Gaussianity of $\inner{v_w}{c}$ we have that $$\Pr[|\langle v_w,c\rangle|\ge \eta \log  n\mid  s_w]\le 2\exp(-\Omega(\frac{1}{4}\log^2 n/\kappa^2)) = \exp(-\Omega(\log^2 n)),$$
		where $\Omega(\cdot)$ hides the dependency on $\kappa$ which is treated as absolute constants. 
		Taking expectations over $s_w$, we obtain that  $$\Pr[\mathcal{F}_w ]  = \Pr[|\langle v_w,c\rangle|\le \frac 1 2 \log  n]\ge 1- \exp(-\Omega(\log^2 n)).$$
		 Note that by definition, we in particular have that conditioned on $\mathcal{F}_w$, it holds that $\exp(\inner{v_w}{c})\le \sqrt{n}$.

		Let the random variable $X_w$ have the same distribution as $\exp(\langle v_w,c\rangle) \vert_{\mathcal{F}_w}$. We prove that the random variable $Z'_c = \sum_{w}X_w$ concentrates well.  By convexity of the exponential function, we have that the mean of $Z_c'$ is lowerbounded
		
		$$\Exp[Z_c'] =n\Exp\left[\exp(\langle v_w,c\rangle)\vert_{\mathcal{F}_w}\right]\ge n\exp( \Exp \left[\langle v_w,c\rangle \vert_{\mathcal{F}_w} \right])= n$$
		
		and the variance is upperbounded by 
		
		\begin{align*}
		\mathbb{V} [Z_c'] &\le n\Exp\left[\exp(\langle v_w,c\rangle)^2\vert_{\mathcal{F}_w}\right] \\
		&\le \frac{1}{\Pr[\mathcal{F}_w]} \Exp\left[\exp(\langle v_w,c\rangle)^2\right] \\
		&\le  \frac{1}{\Pr[\mathcal{F}_w]}\Lambda n \le 1.1 \Lambda n
		\end{align*}
		
		where the second line uses the fact that 
		\begin{eqnarray*}
		\Exp\left[\exp(\langle v_w,c\rangle)^2\right]  &&= \Pr[\mathcal{F}_w]\Exp\left[\exp(\langle v_w,c\rangle)^2\vert {\mathcal{F}_w}\right] + \Pr[\overline{\mathcal{F}}_w]\Exp\left[\exp(\langle v_w,c\rangle)^2\vert {\overline{\mathcal{F}}_w}\right] \\
		&& \ge \Pr[\mathcal{F}_w]\Exp\left[\exp(\langle v_w,c\rangle)^2\vert {\mathcal{F}_w}\right]. 
		\end{eqnarray*}
		
		Moreover, by definition, for any $w$, $|X_w|\le \sqrt{n}$. Therefore by Bernstein's inequality, we have that 
		$$\Pr\left[|Z_c' -\Exp[Z_c']| > \epsilon n \right]\le \exp(-\frac{\frac{1}{2}\epsilon^2 n^2}{1.1\Lambda n + \frac{1}{3}\sqrt{n}∑\cdot \epsilon n})$$

		Note that $\Exp[Z_c']\ge n$, therefore for $\epsilon \gg \frac{\log^2 n}{\sqrt{n}}$, we have, 
				\begin{align}
					\Pr\left[|Z_c' -\Exp[Z_c']| > \epsilon \Exp[Z_c']\right] & \le \Pr\left[|Z_c' -\Exp[Z_c']| > \epsilon n \right] \le \exp(-\frac{\frac{1}{2}\epsilon^2 n^2}{\Lambda n + \frac{1}{3}\sqrt{n}\cdot \epsilon n})\nonumber \\
					& \le \exp(-\Omega(\min\{  \epsilon^2n/\Lambda, \epsilon \sqrt{n}\})) \nonumber\\
					& \le \exp(-\Omega(\log^2 n)) \nonumber
				\end{align}
				
				
		Let $\mathcal{F} = \cup_w \mathcal{F}_w$ be the union of all $\mathcal{F}_w$. Then by union bound, it holds that $\Pr[\bar{\mathcal{F}}]\le \sum_{w}\Pr[\bar{\mathcal{F}_w}] \le n\cdot \exp(-\Omega(\log^2n)) = \exp(-\Omega(\log^2n))$.  We have that by definition, $Z_c'$ has the same distribution as $Z_c\vert_{\mathcal{F}}$. Therefore, we have that 
		
		\begin{equation}
		\Pr[|Z_c -\Exp[Z_c']| > \epsilon \Exp[Z_c']\mid \mathcal{F}]\le \exp(-\Omega(\log^2 n)) \label{eqn:eqn12}
		\end{equation}
		
		and therefore
				\begin{align}
					\Pr[|Z_c -\Exp[Z_c']| > \epsilon  \Exp[Z_c']] &= \Pr[\mathcal{F}]\cdot \Pr[|Z_c -\Exp[Z_c']| > \epsilon  \Exp[Z_c']\mid \mathcal{F}] + \Pr[\bar{\mathcal{F}}]  \Pr[|Z_c -\Exp[Z_c']| > \epsilon  \Exp[Z_c']\mid \bar{\mathcal{F}}] \nonumber\\
					& \le \Pr[|Z_c -\Exp[Z_c']| > \epsilon  \Exp[Z_c']\mid \mathcal{F}]  + \Pr[\bar{\mathcal{F}}] \nonumber\\
					& \le \exp(-\Omega(\log^2n)) \label{eqn:eqn13}
					\end{align}
			where at the last line we used the fact that $\Pr[\bar{\mathcal{F}}]\le\exp(-\Omega(\log^2n))$ and equation~\eqref{eqn:eqn12}. 
			
			Let $Z =  \Exp[Z_c'] =  \Exp[\exp(\inner{v_w}{c})\mid |\inner{v_w}{c}| < \frac{1}{2}\log n]$ (note that $\Exp[Z_c']$ only depends on the norm of $\|c\|$ which is equal to 1). Therefore we obtain that with high probability over the randomness of $v_w$'s, 
			
			\begin{equation}
				(1-\epsilon_z)Z\le Z_c\le (1+\epsilon_z)Z\label{eqn:lem_zc_proof}
			\end{equation}
			Taking expectation over the randomness of $c$, we have that 
			\begin{equation*}
			\Pr_{c,v_w}[\textrm{~\eqref{eqn:lem_zc_proof} holds} ]\ge 1-\exp(-\Omega(\log^2 n))
			\end{equation*}
			
			Therefore by a standard averaging argument (using Markov inequality), we have 
			\begin{equation*}
			\Pr_{v_w}\left[\Pr_c[\textrm{~\eqref{eqn:lem_zc_proof} holds} ]\ge 1-\exp(-\Omega(\log^2 n))\right] \ge 1-\exp(-\Omega(\log^2 n))
			\end{equation*}
			For now on we fix a choice of $v_w$'s so that $\Pr_c[\textrm{~\eqref{eqn:lem_zc_proof} holds} ]\ge 1-\exp(-\Omega(\log^2 n))$ is true. Therefore in the rest of the proof, only $c$ is random variable, and with probability $1 - \exp(-\Omega(\log^2 n))$ over the randomness of $c$, it holds that, 
			\begin{equation}
			(1-\epsilon_z)Z\le Z_c\le (1+\epsilon_z)Z. \label{eqn:good}
			\end{equation}
	\end{proof}
%

\section{Maximum likelihood estimator for co-occurrence} \label{sec:mlecooc}

Let $\LengthDocument$ be the corpus size, and $X_{w,w'}$ the number of times words $w, w'$ co-occur within a context of size $10$ in the corpus. According to the model, the probability of this event at any particular
time is $\log p(w,w')  \propto \nbr{v_{w} +  v_{w'}}_2^2.$ Successive samples from a random walk are not 
independent of course, but if the random walk mixes fairly quickly (and the mixing time of our random walk
is related to the {\em logarithm} of the number of words) then the 
set of $X_{w,w'}$'s over all word pairs is distributed  up to a very close approximation as a multinomial distribution $\mul(\tilde{L}, \{p(w,w')\})$ where $\tilde{L} = \sum_{w,w'} X_{w,w'}$ is the total number of word pairs in consideration (roughly $10 \LengthDocument$). 

Assuming this approximation, we show below that the maximum likelihood values for the word vectors correspond to the following optimization, 
\begin{align*}  
\min_{\cbr{v_w}, C} & \sum_{w,w'} X_{w,w'} \rbr{ \log(X_{w,w'}) - \nbr{v_{w}\!+\! v_{w'}}_2^2 - C}^2   \mbox{\qquad\qquad\qquad\qquad (Objective {\bf SN})}
\end{align*} 

Now we give the derivation of the objective.
According to the multinomial distribution, maximizing the likelihood of $\{X_{w,w'}\}$ is equivalent to maximizing
\begin{align*}
\ell & = \log\left(\prod_{(w, w')} p(w, w')^{X_{w,w'}} \right) = \sum_{(w, w')} X_{w,w'} \log p(w, w').
\end{align*}

To reason about the likelihood, denote the logarithm of the ratio between the expected count and the empirical count as
\begin{align*}
\Delta_{w,w'} = \log\left(\frac{\tilde{L} p(w, w')  }{ X_{w,w'} } \right).   
\end{align*}
Note that 
\begin{align}
 \ell & = \sum_{(w, w')} X_{w,w'}\log p(w, w')  \nonumber \\
& =  \sum_{(w, w')}  X_{w,w'} \left[\log\frac{X_{w,w'}}{\tilde{L}}  +   \log\left(\frac{\tilde{L} p(w, w')  }{ X_{w,w'} } \right) \right]  \nonumber \\
& =  \sum_{(w, w')}  X_{w,w'}  \log\frac{X_{w,w'}}{\tilde{L}}  +   \sum_{(w, w')}  X_{w,w'}  \log\left(\frac{\tilde{L} p(w, w')  }{ X_{w,w'} } \right)  \nonumber\\
& =  c  +   \sum_{(w, w')}   X_{w,w'}  \Delta_{w,w'} \label{eqn:likelihood_app}
\end{align}
where we let $c$ denote the constant $\sum_{(w, w')}  X_{w,w'}  \log\frac{X_{w,w'}}{\tilde{L}}$. 
Furthermore, we have 
\begin{align*}
\tilde{L} 
& = \sum_{(w,w')} \tilde{L} p_{w,w'} \\
& = \sum_{(w,w')} X_{w,w'} e^{\Delta_{w,w'} } \\
& = \sum_{(w,w')}  X_{w,w'} (1 + \Delta_{w,w'} + \Delta_{w,w'}^2/2 + O(|\Delta_{w,w'}|^3) )
\end{align*}
and also $\tilde{L} = \sum_{(w,w')} X_{w,w'}$. So
\begin{align*}
 \sum_{(w,w')}  X_{w,w'} \Delta_{w,w'}  =  - \left(\sum_{(w,w')}  X_{w,w'} \Delta_{w,w'}^2/2 + \sum_{(w,w')}  X_{w,w'} O(|\Delta_{w,w'}|^3) \right).
\end{align*}
Plugging this into (\ref{eqn:likelihood}) leads to
\begin{align}
c - \ell  = \sum_{(w,w')}  X_{w,w'} \Delta_{w,w'}^2/2 + \sum_{(w,w')}  X_{w,w'} O(|\Delta_{w,w'}|^3). \label{eqn:finalexpresion_app}
\end{align}

When the last term is much smaller than the first term on the right hand side, maximizing the likelihood is approximately equivalent to minimizing the first term on the right hand side, which is our objective:
\begin{align*}
\sum_{(w, w')}   X_{w,w'}  \Delta_{w,w'}^2 \approx \sum_{(w, w')}   X_{w,w'} \left(\|v_w + v_{w'} \|_2^2/(2d) - \log X_{w,w'} + \log\tilde{L} - 2\log Z \right)^2
\end{align*}
where $Z$ is the partition function.

We now argue that the last term is much smaller than the first term on the right hand side in (\ref{eqn:finalexpresion_app}). For a large $X_{w,w'}$, the $\Delta_{w,w'}$  is close to 0 and thus the induced approximation error is small. Small $X_{w,w'}$'s only contribute a small fraction of the final objective (\ref{eqn:finalexpresion}), so we can safely ignore the errors. 
To see this, note that the objective $\sum_{(w,w')}  X_{w,w'} \Delta_{w,w'}^2$ and the error term $\sum_{(w,w')}  X_{w,w'} O(|\Delta_{w,w'}|^3)$ differ by a factor of  $|\Delta_{w,w'}|$ for each $X_{w,w'}$. 
For large $X_{w,w'}$'s,  $|\Delta_{w,w'}| \ll 1$, and thus their corresponding errors are much smaller than the objective. So we only need to consider the $X_{w,w'}$'s that are small constants. The co-occurrence counts obey a power law distribution (see, e.g.~\cite{pennington2014glove}). That is, if one sorts $\{X_{w,w'}\}$ in decreasing order, then the $r$-th value in the list is roughly
$$
  x_{[r]} = \frac{k}{r^{5/4}}
$$ 
where $k$ is some constant. Some calculation shows that 
$$
  \tilde{L} \approx 4k, ~~ \sum_{X_{w,w'} \le x} X_{w,w'} \approx 4 k^{4/5} x^{1/5},
$$
and thus when $x$ is a small constant 
$$
  \frac{\sum_{X_{w,w'} \le x} X_{w,w'} }{\tilde{L}} \approx \left(\frac{4x}{\tilde{L}}\right)^{1/5} = O\left(\frac{1}{\tilde{L}^{1/5}} \right).
$$
So there are only a negligible mass of $X_{w,w'}$'s that are small constants, which vanishes when $\tilde{L}$ increases. Furthermore, we empirically observe that the relative error of our objective is $5\%$, which means that the errors induced by $X_{w,w'}$'s that are small constants is only a small fraction of the objective. Therefore, $\sum_{w,w'} X_{w,w'} O(|\Delta_{w,w'}|^3)$ is small compared to the objective and can be safely ignored.

\end{document}